\documentclass[sigconf]{acmart}
\AtBeginDocument{%
  }

\copyrightyear{2026}
\acmYear{2026}
\setcopyright{cc}
\setcctype{by}
\acmConference[WWW '26] {Proceedings of the ACM Web Conference 2026}{April 13--17, 2026}{Dubai, United Arab Emirates.}
\acmBooktitle{Proceedings of the ACM Web Conference 2026 (WWW '26), April 13--17, 2026, Dubai, United Arab Emirates}
\acmISBN{979-8-4007-2307-0/2026/04}
\acmDOI{10.1145/3774904.3792295}

\usepackage{tikz}
\usepackage{xcolor}
\usepackage{amsmath,amsfonts,amsthm}
\usepackage[ruled,vlined,linesnumbered]{algorithm2e}
\usepackage{algorithmic}

\SetKwComment{Comment}{/* }{ */}

\usepackage{wrapfig}
\usepackage{subcaption}  
\usepackage{adjustbox}    
\usepackage{booktabs}
\usepackage{tabulary}
\usepackage{multirow}
\usepackage{enumitem}
\usepackage{graphicx}
\usepackage{appendix}
\usepackage{rotating}

\usepackage{balance} 

\usepackage[font=small,skip=0pt]{caption}
\usepackage{tabularx}
\usepackage[table]{xcolor}
\microtypecontext{spacing=nonfrench}
\usepackage[normalem]{ulem}
\usepackage{cleveref}

\newtheorem{definition}{Definition}
\newtheorem{assumption}{Assumption}
\newtheorem{theorem}{Theorem}

\begin{document}
\title{WinFLoRA: Incentivizing Client-Adaptive Aggregation in Federated LoRA under Privacy Heterogeneity}

\author{Mengsha Kou}
\affiliation{%
  \institution{RMIT University}
  \city{Melbourne}
  \country{Australia}}
\email{mengsha.kou@student.rmit.edu.au}

\author{Xiaoyu Xia}
\authornote{Corresponding author.}
\affiliation{%
  \institution{RMIT University}
  \city{Melbourne}
  \country{Australia}}
\email{xiaoyu.xia@rmit.edu.au}

\author{Ziqi Wang}
\affiliation{%
  \institution{RMIT University}
  \city{Melbourne}
  \country{Australia}}
\email{ziqi.wang6@student.rmit.edu.au}

\author{Ibrahim Khalil}
\affiliation{%
  \institution{RMIT University}
  \city{Melbourne}
  \country{Australia}}
\email{Ibrahim.khalil@rmit.edu.au}

\author{Runkun Luo}
\affiliation{%
  \institution{Huazhong University of Science and Technology}
  \city{Wuhan}
  \country{China}}
\email{rkluo@hust.edu.cn}

\author{Jingwen Zhou}
\affiliation{%
  \institution{CSIRO’s Data61}
  \city{Melbourne}
  \country{Australia}}
\email{helen.zhou@data61.csiro.au}

\author{Minhui Xue}
\affiliation{%
  \institution{CSIRO’s Data61 and Responsible AI Research (RAIR) Centre}
  \city{Adelaide}
  \country{Australia}}
\email{jason.xue@data61.csiro.au}

\renewcommand{\shortauthors}{Mengsha Kou et al.}

\begin{abstract}
Large Language Models (LLMs) increasingly underpin intelligent web applications, from chatbots to search and recommendation, where efficient specialization is essential. Low-Rank Adaptation (LoRA) enables such adaptation with minimal overhead, while federated LoRA allows web service providers to fine-tune shared models without data sharing. However, in privacy-sensitive deployments, clients inject varying levels of differential privacy (DP) noise, creating privacy heterogeneity that misaligns individual incentives and global performance. 
In this paper, we propose WinFLoRA, a privacy-heterogeneous federated LoRA that utilizes aggregation weights as incentives with noise awareness. Specifically, the noises from clients are estimated based on the uploaded LoRA adapters.
A larger weight indicates greater influence on the global model and better downstream task performance, rewarding lower-noise contributions.
By up-weighting low-noise updates, WinFLoRA improves global accuracy while accommodating clients' heterogeneous privacy requirements. 
Consequently, WinFLoRA aligns heterogeneous client utility in terms of privacy and downstream performance with global model objectives without third-party involvement. Extensive evaluations demonstrate that across multiple LLMs and datasets, WinFLoRA achieves up to 52.58\% higher global accuracy and up to $2.56\times$ client utility than state-of-the-art benchmarks. Source code is publicly available at \url{https://github.com/koums24/WinFLoRA.git}.
\end{abstract}

\ccsdesc[500]{Security and privacy~Distributed systems security}

\keywords{Federated fine-tuning; privacy preservation; incentive mechanism.}
\maketitle

\section{Introduction}
Large language models (LLMs) demonstrate strong proficiency across a spectrum of downstream tasks~\cite{achiam2023gpt,team2023gemini,touvron2023llama}, and increasingly power web services such as chatbots, writing assistants, and search engines~\cite{wu2023bloomberggpt}. Fine-tuning is a practical way to specialize a pre-trained LLM for a downstream task~\cite{dodge2020fine}. However, full fine-tuning is computationally expensive. Parameter-efficient fine-tuning (PEFT) reduces this cost by freezing the backbone model and learning a small set of task-specific parameters~\cite{houlsby2019parameter,hu2022lora}. Among PEFT methods, Low-Rank Adaptation (LoRA) is widely adopted for web services due to low memory consumption and computation overheads, as it trains low-rank adapters by less than 5\% of the parameters while achieving comparable accuracy~\cite{hu2022lora}.

In data-siloed settings, web services can collaborate through federated fine-tuning (FFT) to jointly fine-tune a shared backbone without exposing their raw data to improve the model performance~\cite{babakniya2023slora,sun2024improving,kuo2024federated}.
However, this federated paradigm is still vulnerable to privacy leakage, such as membership inference~\cite{wen2024membership,xia2025edge} and data reconstruction~\cite{lukas2023analyzing,fu2023practical}, which are particularly concerning in sensitive domains such as telemedicine and fintech~\cite{wang2024gees,tang2024merit,wang2025MoSEEC}. 
Differential privacy (DP) provides a feasible defense by enabling clients to add calibrated Gaussian noise to local data or LoRA updates during training, thereby reducing privacy leakage~\cite{dwork2006calibrating,wang2025KGEES,wang2024balancing}. 

\begin{figure}[]
    \centering
    \includegraphics[width=\linewidth]{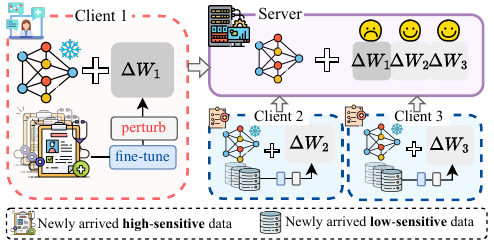}
    \caption{Clients apply different levels of noise based on data sensitivity, with higher-sensitivity clients adding stronger perturbations before uploading updates over time.}
    \label{fig:intuition}
\end{figure}

Consider an organization of telemedicine clients that collaboratively fine-tunes a shared LLM on their own clinical records to improve downstream performance, as illustrated in ~\Cref{fig:intuition}. As new clinical data arrive, the collaboration proceeds continuously, with clients fine-tuning their own LoRA adapters on the latest local data while injecting noise to meet their privacy requirements. In each round, these adapters are uploaded and aggregated to update the global model. Clients independently choose their perturbation strategy, with stronger privacy preferences resulting in larger noise. However, this noise degrades update quality and reduces global accuracy, creating a fundamental tension between preserving client-side privacy and maintaining overall model performance~\cite{birrell2024differentially}.

To mitigate the global performance degradation caused by self-interested behaviors, existing studies aim to incentivize clients to trade privacy for accuracy through external compensation such as monetary payments~\cite{mao2024game,wang2022privaim,hu2022incentive}. These approaches often overlook privacy heterogeneity and assume the presence of a trusted third party for incentive administration, an assumption that may not hold in decentralized or sensitive settings~\cite{zhang2021incentive}. Furthermore, pricing updates by Shapley values or data valuation often require repeated retraining, which becomes infeasible at LLM scale due to high computational cost~\cite{chen2024fdfl}. Thus, aligning individual behavior with the global performance objective remains a challenge under privacy heterogeneity.

In this paper, we propose \textbf{WinFLoRA}, a \underline{w}eight-based \underline{in}centive for \underline{f}ederated \underline{LoRA} under privacy heterogeneity. The server estimates each client’s noise directly from the uploaded LoRA adapters and assigns a weight based on the noise estimation. The key idea is to assign lower aggregation weights to LoRA adapters with higher levels of injected noise, thereby mitigating their negative impact on global model performance. In contrast, lower-noise updates receive greater weight, increasing their influence on aggregation and potentially improving their own downstream task performance. This acts as an incentive mechanism, rewarding higher-quality (i.e., less perturbed) contributions. Meanwhile, privacy-prioritizing clients can retain stronger noise with reduced aggregation weight, yet still benefit from a high-quality global model. This weighting mechanism ties each client’s reward to the quality of their update, aligning individual incentives with global performance. Moreover, WinFLoRA infers noise levels directly from LoRA adapters rather than relying on self-reported privacy preferences, requiring neither explicit payments nor a trusted third party. This enhances its practicality in real-world settings. The main contributions are summarized as follows:
\begin{itemize}[leftmargin=*,noitemsep]
    \item To the best of our knowledge, WinFLoRA is the first approach to significantly improve global model accuracy in federated LoRA while accommodating heterogeneous client privacy levels. It achieves this by aligning individual utility with the system objective through a noise-aware, weight-based incentive mechanism.
    \item We propose a noise-aware weight allocation strategy in WinFLoRA, mapping the injected noise scale to aggregation weights. By allocating more weights to high-quality updates as an incentive, the global performance is significantly improved.
    \item We formalize the interaction among clients induced by noise-aware weighting as a stochastic aggregative game of noise selection, prove the existence of a stationary Markov equilibrium, and empirically show the stability of client strategies.
    \item Extensive experiments are conducted across models, datasets, and system configurations. The results show that WinFLoRA improves global model performance by up to 52.58\%, and average client utility by up to 2.56$\times$, compared to benchmarks.
\end{itemize}

\section{Preliminaries and Related Work}
\subsection{Preliminaries}
\noindent\textbf{LoRA-based Federated fine-tuning (FFT). }
LoRA is a parameter-efficient fine-tuning (PEFT) method that adapts pre-trained models by adding a low-rank adapter to their frozen weights, avoiding updates to most original parameters~\cite{hu2022lora}. For a target module with pre-trained base model $\mathrm{W}\in\mathbb{R}^{d_{in}\times d_{out}}$, LoRA learns two low-rank matrices $\mathrm{B}\in\mathbb{R}^{d_{out}\times r}$ and $\mathrm{A}\in\mathbb{R}^{r\times d_{in}}$ with $r\ll\min(d_{in},d_{out})$, yielding:
\begin{equation}
\label{eq:lora_update}    
    \mathrm{W}' \,=\, \mathrm{W} + \Delta \mathrm{W} \,=\, \mathrm{W} + \mathrm{B}\mathrm{A}
\end{equation}
A common initialization is $\mathrm{A}\sim\mathcal{N}(0,\sigma^2)$ and $\mathrm{B}=0$~\cite{wang2024flora}. 

In the federated setting, 
client $c_i$ produces $\Delta\mathrm{W}_i=\mathrm{B}_i\mathrm{A}_i$ with $\mathrm{B}_i\in\mathbb{R}^{d_{out}\times r_i}$ and $\mathrm{A}_i\in\mathbb{R}^{r_i\times d_{in}}$, where the rank $r_i$ may vary across clients. The server aggregates updates via block stacking: it vertically stacks $\{\mathrm{B}_i\}_{i=1}^N$ to form
$\mathrm{B}_g=[\mathrm{B}_1;\mathrm{B}_2;\cdots;\mathrm{B}_N]\in\mathbb{R}^{d_{out}\times(\sum_{i=1}^N r_i)}$ and horizontally concatenates $\{\mathrm{A}_i\}_{i=1}^N$ to form
$\mathrm{A}_g=[\mathrm{A}_1,\mathrm{A}_2,\cdots,\mathrm{A}_N]\in\mathbb{R}^{(\sum_{i=1}^N r_i)\times d_{in}}$. According to~\cref{eq:lora_update},the global aggregation is:
\begin{equation}
\label{eq:stacking-based}
    \Delta \mathrm{W}_g \,=\, \mathrm{B}_g\mathrm{A}_g \,=\, \sum\nolimits_{i=1}^N \mathrm{B}_i\mathrm{A}_i
\end{equation}
The global model is updated by $\mathrm{W}'=\mathrm{W}+\Delta\mathrm{W}_g$.

\noindent\textbf{Differential privacy. }
Differential privacy provides a formal guarantee that the inclusion or exclusion of any single record has only a limited effect on the distribution of the output~\cite{dwork2006calibrating}. For a mechanism $M$ and two adjacent datasets $D,D'$ that differ in only a single record, $(\varepsilon,\delta)$-DP can be guaranteed when $M$ satisfies:
\begin{equation}
\Pr[M(D)\in \mathcal{O}] \le e^{\varepsilon}\,\Pr[M(D')\in \mathcal{O}] + \delta,
\end{equation}
where $\mathcal{O}$ denotes all possible subsets of outputs of $M$, $\varepsilon$ denotes a distinguishable bound between the output of $M(D)$ and $M(D')$, and $\delta$ represents the probability of information leakage. The Gaussian mechanism is a classical DP mechanism which adds zero-mean Gaussian noise with a certain variance $\sigma^2$ to an existing algorithm $f(\cdot)$ as $M(D)=f(D)+\xi$, where $\xi \sim \mathcal N(\sigma^2 I)$~\cite{wei2020federated,sun2024improving}. In this paper, DP noise is injected into LoRA adapters, i.e., ($A_i+\xi_{i,A}$,$B_i+\xi_{i,B}$), where $\xi_{A,i}\!\sim\!\mathcal{N}(0,\sigma_{A,i}^2 I)$ and$\xi_{B,i}\!\sim\!\mathcal{N}(0,\sigma_{B,i}^2 I)$, in FFT to mitigate the training-data leakage, which is widely used~\cite{wei2020federated}.

\subsection{Related Work}
\noindent\textbf{Parameter-efficient fine-tuning (PEFT). }
Full-parameter fine-tuning for LLMs imposes significant computational and storage overhead. For instance, even a 7B-parameter model requires 28GB of GPU for weights alone, while gradients and optimizer states typically add a further 92GB, bringing the total to 120GB~\cite{rajbhandari2020zero}. The high cost of full fine-tuning motivates PEFT, such as prompt tuning, prefix tuning, adapter tuning, and LoRA~\cite{ding2022delta}. 
Prompt tuning and prefix tuning learn a sequence of trainable prompts or task-specific vectors in the input space to steer the model’s behavior on downstream tasks while keeping model weights frozen~\cite{houlsby2019parameter,li2021prefix}.
Adapter tuning inserts lightweight trainable modules between Transformer layers and fine-tunes only these adapters~\cite{cai2023efficient}. 
LoRA decomposes weight updates into a pair of low-rank matrices injected in parallel with frozen weights without modifying backbone architectures~\cite{hu2022lora}. Several studies have demonstrated the better flexibility and superior performance of LoRA than adapter-based methods~\cite{babakniya2023slora,hu2023llm}. Thus, we adopt LoRA as the PEFT method in our work. 

\noindent\textbf{FFT with LLMs.}
Although fine-tuned LLMs often serve as backbones across application domains, effective fine-tuning still relies on large, domain-specific datasets. FFT offers a practical way to fine-tune across decentralized datasets held by multiple parties~\cite{babakniya2023slora}. FedIT is the first attempt to integrate LoRA-based PEFT into the FL framework using vanilla FedAvg~\cite{mcmahan2017communication}, but it exhibits limitations under heterogeneous LoRA configurations~\cite{zhang2024towards}. Subsequent approaches mitigate heterogeneity by zero-padding LoRA modules at the cost of additional computation and communication, yet still incur aggregation noise~\cite{chen2024rbla,gao2025federated}. Flora eliminates the noise by aggregating by matrix stacking to form the global low-rank update, while accommodating heterogeneous client-specific ranks~\cite{wang2024flora}. Other lines of work enhance LoRA’s utility in FL through sparse adaptation space design~\cite{kuo2024federated}, SVD-based mitigation of data heterogeneity~\cite{yan2024federa}, support for multimodal inputs~\cite{chen2024feddat}, and communication -efficient optimization~\cite{nguyen2024towards}. Despite the improvement in efficiency and generalization, FFT remains vulnerable to privacy threats, such as membership inference attacks and data reconstruction attacks. A common way to mitigate privacy risks is to incorporate differential privacy (DP) into FFT~\cite{xu2025dp,sun2024improving}. However, the added noises lead to model performance degradation, especially under privacy heterogeneity.

\noindent\textbf{Incentive mechanisms in FFT. }
Incentivizing participation with high-quality contributions is a longstanding challenge in FL. Existing studies fall into three categories, including contract-based, contribution-based, and reputation-based mechanisms. Contract-based mechanisms model server-client interaction as a Stackelberg or contract-theoretic game and set prices to generate effort under information asymmetry~\cite{zhang2021incentive,mao2024game,lin2023heterogeneous,yang2023csra}. In practice, such mechanisms rely on self-reported private information or effort indicators, which may compromise privacy and incentivize dishonest behavior, and often require explicit payments to clients. Contribution-based reward mechanisms estimate each client’s marginal value, typically via Shapley-value attribution, which is computationally expensive and scales poorly with LLMs~\cite{luo2023incentive,chen2024fairreward}. Reputation-based mechanisms prioritize clients based on the reliability of their past updates or resource availability, and are sometimes combined with blockchain-based auditing. However, such integration introduces significant operational complexity in ensuring truthful contribution valuation~\cite{zhang2021incentive,rashid2025trustworthy}. 

\section{Problem Formulation}
\subsection{System Model}
Considering the continual FFT over $T$ rounds with a set of clients $\mathcal{C}=\{c_i\}_{i=1}^N$. In round $t$, each client $c_i$ trains her low-rank adapters $(A_i^{t}, B_i^{t})$ on newly arrived data $D_i^{t}$, while keeping the pre-trained weights frozen. Before upload, $c_i$ optionally injects local privacy noise with scale $\sigma_i^{t}$ to her adapter update, based on her privacy preference. The server then executes one-shot aggregation of LoRA updates from all the clients in each round to generate the global model $W_g^{t}$~\cite{wang2024one}. As new data continue to arrive, this process iterates over time. The main notations and definitions are summarized in ~\Cref{tab:summary_of_notations} in \Cref{appendix:notations}.

\noindent\textbf{System Utility}.
The system aims to gain a high-accuracy global model from decentralized data.
In round $t$, the system utility is the performance of the aggregated global model, calculated with:
\begin{equation}
U_s^t \;=\; \mathcal{A}(W_g^t,\, \mathcal{D}^t;\boldsymbol{\sigma}^t),
\end{equation}
where $W_g^t$ denotes the global parameters obtained by aggregating all clients’ LoRA updates in round $t$, $\boldsymbol{\sigma}^t=\{\sigma_i^t\}_{i=1}^N$ is the client-side noise scales, and $\mathcal{D}^t$ is the global dataset across all clients. The mapping $\mathcal{A}(\cdot)$ denotes a performance metric, e.g., accuracy or negative loss, with larger values indicating better performance.
\begin{figure*}[]
    \centering
    \includegraphics[width=0.95\linewidth]{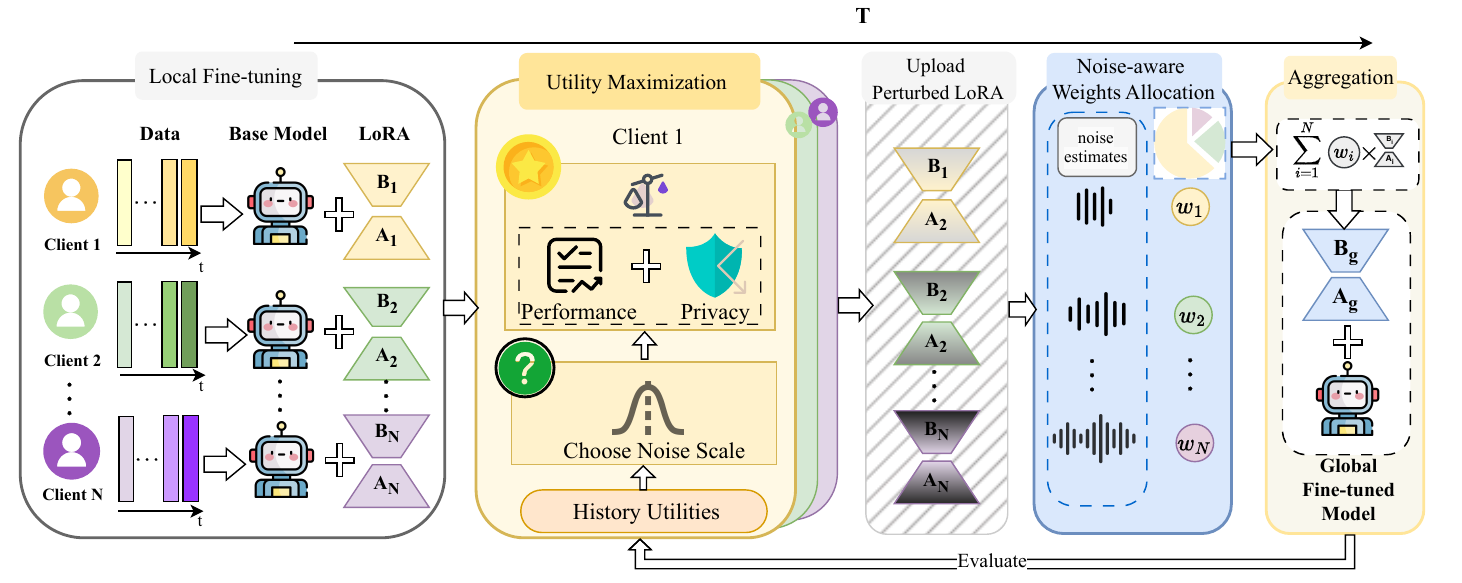}
    \caption{System overview of WinFLoRA. In each round, clients train their own LoRA with newly arrived local data, inject client-specific noise, and upload perturbed adapters. The server estimates per-client noise and assigns weights inversely, then aggregates the weighted updates to form a global fine-tuned model.}
    \label{fig:overview}
\end{figure*}

\noindent\textbf{Client Utility}. 
The client primarily aims to improve downstream task performance while also seeking a level of privacy protection based on her data sensitivity. Accordingly, client $c_i$'s per-round utility consists of downstream task performance and privacy benefit, calculated with:
\begin{equation}
\label{eq:client_utility}
 U_i^t \;=\; \mathcal{G}_i(W_g^t, \mathcal{D}_i^t;\boldsymbol{\sigma}^t)\;+\;\gamma_i\cdot \mathcal{P}_i(\sigma_i^t),
\end{equation}
where $\mathcal{G}_i(\cdot)$ measures the performance that client $c_i$ derives from the current global model $W_g^t$ on her local data $D_i^t$, given the joint noise profile $\boldsymbol{\sigma}^t$. The privacy term is normalized as
$\mathcal{P}_i(\sigma_i^t)=\sigma_i^t/\sigma_{\max}\in[0,1]$, increasing monotonically with the local noise scale $\sigma_i^t$. The coefficient $\gamma_i>0$ captures the preference for privacy of the client $c_i$, and a larger $\gamma_i$ implies a higher preference for privacy benefit relative to performance. This shows a trade-off that a larger $\sigma_i^t$ raises $\mathcal{P}_i(\cdot)$ but degrades the quality of LoRA updates, potentially reducing $\mathcal{G}_i(\cdot)$.

\subsection{Problem Statement}
In continual FFT, client $c_i$ selects a noise scale $\sigma_i^t\in[0,\sigma_{\max}]$ in each round $t$ to maximize her long-term average individual utility. Since the server aggregates heterogeneous updates, higher noise lowers the aggregate model quality and can reduce other clients’ rewards, creating a negative externality. When privacy preferences vary, clients tend to internalize their own privacy gains while externalizing the accuracy loss to others, often resulting in excessive noise and global performance degradation. To address this issue, we design an incentive mechanism $\mathcal{M}$ that aligns individually optimal noise choices with the long-term system objective in a Win-Win mode. We formulate the problem as:
\begin{align}
\max_{\mathcal M} \lim_{T\to\infty}\frac{1}{T}\sum_{t=1}^T U_s^t \label{eq:sys-avg-obj},\quad
\max_{\mathcal M}\lim_{T\to\infty}\frac{1}{T}\sum_{t=1}^T U_i^t, \forall c_i \in \mathcal C
\end{align}

\section{System Design}
\subsection{WinFLoRA Overview}
\Cref{fig:overview} illustrates the system overview of WinFLoRA. The server utilizes noise-aware weighting incentive (\Cref{subsec:NWA}) that maps per-client noise estimates to aggregation weights, involving clients with heterogeneous privacy preferences (\Cref{subsec:INA}) into a utility maximization game(\Cref{subsec:theory}).
Specifically, on the server side, the server estimates noise scale directly from uploaded adapters, computes aggregation weights that sum to one and are inversely proportional to the estimated noise, so cleaner updates receive larger weights and noisier updates receive smaller weights (\Cref{subsec:NWA}). Then, it aggregates the weighted updates $\{\Delta W_i^t\}_{i=1}^N$ with the assigned weights $\{w_i^t\}_{i=1}^N$ to obtain the global update $\Delta W_g^t$ and merges it into the base model to form the global fine-tuned model.

Given this incentive, clients act selfishly to maximize their own utility based on their performance–privacy preference(\Cref{subsec:INA}). In round $t$, each client $c_i$ performs local LoRA-based fine-tuning to produce $A_i$ and $B_i$, selects a noise scale $\sigma_i^t$ to perturb her LoRA updates by learning from history utility records. Since the global model aggregates all updates, the noise injected by a single client can impact the accuracy of other clients. Thus, each client’s reward is also impacted by the choices of others. With the continuous arrival of new data, clients’ interdependent decisions naturally constitute a repeated aggregative game of utility maximization.

\subsection{Noise-aware Weights Allocation (NWA)}
\label{subsec:NWA}
\noindent\textbf{Noise estimation.} To avoid eliciting privacy preferences, the server estimates each client's noise scale directly from their uploaded LoRA updates $\{\Delta W_i=B_iA_i\}_{i=1}^N$ using the leave-one-out principal component analysis (LOO-PCA), inspired by the leave-one-out comparison that decouples the dependence between cluster structure and noise in spectral clustering~\cite{zhang2024leave}. Concretely, we form a public subspace of all clients' updates and use the residual outside the subspace to obtain estimated noise scales $\{\hat{\sigma}_i\}_{i=1}^N$. The round index $t$ is omitted for brevity here. As illustrated in~\Cref{fig:loo-pca} and~\Cref{alg:loo_pca_b}, taking LoRA $B$ matrices as an example, for client $c_i$, the server forms a matrix $X^{(-i)}$ by stacking the flattened LoRA $B$ matrices from all clients except $c_i$ (lines~\ref{line:stack}--\ref{line:formX}). Then it computes the row-wise mean $\mu^{(-i)}$ of $X^{(-i)}$ and centers it as $X^{(-i)}\!\leftarrow\!X^{(-i)}-\mu^{(-i)}\mathbf{1}^\top$ (lines~\ref{line:mu}---\ref{line:center}). 
By applying singular value decomposition (SVD) to the centered matrix,
$X^{(-i)}=U_{-i}\Sigma_{-i}V_{-i}^\top$, the server takes the top-$K$ left singular vectors $U^{(-i)}_{K}$ to span a public subspace that captures the shared directions across clients and its projection operator $P_i \;=\; U^{(-i)}_{K}\big(U^{(-i)}_{K}\big)^\top$ that maps any vector to its closest point in that subspace (lines~\ref{line:SVD}--\ref{line:Proj}). For $c_i$'s centered vector $\tilde x_i=x_i-\mu^{(-i)}$, its projection $\hat x_i=P_i\tilde x_i$ represents the component aligned with the public subspace, while the residual $r_i=(I-P_i)\tilde x_i$ lies in the orthogonal complement, where $I$ is the identity matrix (line~\ref{line:residual}). The public subspace concentrates the shared signal, and the orthogonal residual predominantly reflects client-specific noise. Therefore, the noise scale is estimated by a normalized residual energy (line~\ref{line:estimates}):
\begin{equation}
\hat{\sigma}_i^2 \;=\; \frac{\|r_i\|_2^2}{\max(d_{B}-K,\,1)}\,,\qquad
\hat{\sigma}_i \;=\; \sqrt{\hat{\sigma}_i^2}\,,
\end{equation}
where $d_B$ is the dimension of the concatenated and vectorized LoRA $B$ matrices.
The higher residual energy norm $\|r_i\|_2^2$ implies higher injected noise. Details are illustrated in~\Cref{alg:loo_pca_b} in~\Cref{appendix:alg}.

The task-relevant signal that recurs between clients is concentrated in a low-dimensional public subspace~\cite{ding2025sulora}. By contrast, the additive Gaussian noise is zero-mean, independent, and isotropic without directional preference. Thus, its expected projection onto the subspace is zero. Consequently, the noise is predominantly captured by the orthogonal complement of the subspace~\cite{li2024federated}, i.e., the residual energy $||r_i||_2^2$. The same procedure applies to LoRA~$A$. Moreover, since aggregation weight is proportional to relative noise level, the estimate ${\hat\sigma_i}$ preserves the same ranking as the true noise scales.
\begin{figure}[h]
    \centering
    \includegraphics[width=\linewidth]{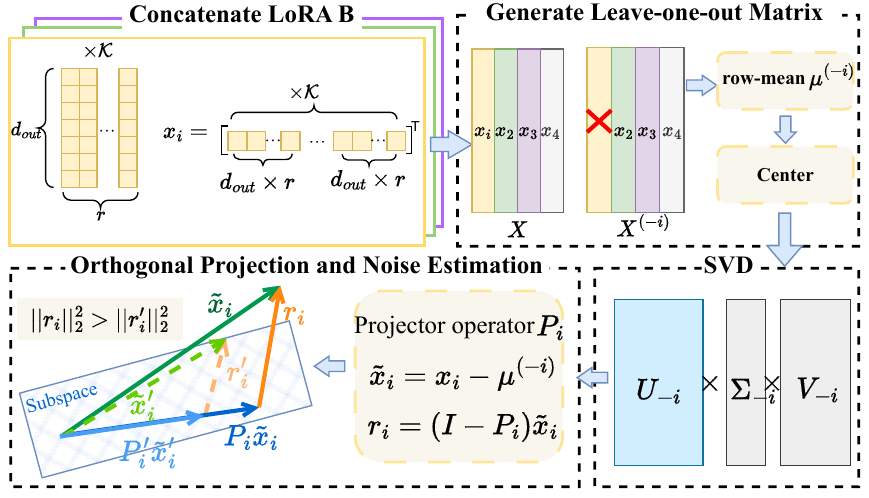}
    \caption{Example of LOO–PCA residual noise estimation. This figure illustrates the per-client noise estimation procedure on LoRA $B$. A larger residual $||r_i||_2^2$ indicates a larger estimated noise.}
    \label{fig:loo-pca}
\end{figure}

\noindent \textbf{Inverse-noise weight incentives.} As detailed in~\Cref{alg:noise_aware_allocation} in~\Cref{appendix:alg}, the server allocates aggregation weights based on each client’s estimated noise scale $\hat{\sigma}_i^t$ to up-weight lower-noise contributions (line~\ref{line:loo}). Specifically, it applies a monotone inverse-noise mapping $s_{i}^t \leftarrow \big(1/(\hat\sigma_{i}^t+\tau)\big)$ (line~\ref{line:score}) with $\tau=1\times10^{-8}$ for numerical stability and then normalizes $w_{i}^t \leftarrow s_{i}^t/\sum_{i=1}^{N} s_{i}^t$ for aggregation (lines~\ref{line:weights}--~\ref{line:aggre}). By prioritizing low-noise updates, the aggregator down-weights noisy contributions and mitigates their adverse impact on the global model, thereby improving global performance. The weight also acts as an incentive, as a larger $w_{i}^t$ translates into greater influence on the aggregated model and, in practice, better adaptation to the client’s downstream task.

\noindent\textbf{Remark.} NWA respects heterogeneous privacy preferences. Clients with stronger privacy preferences rationally maintain higher noise and accept smaller aggregation weights, whereas performance-oriented clients voluntarily de-noise to gain influence. Consequently, the inverse noise weight acts simultaneously as a noise-robust aggregator and as an incentive that aligns the client utilities with the system utility.

\subsection{Individual Noise Adaptation (INA)}
\label{subsec:INA}
In FFT, each client $c_i$ aims to maximize its long-term average utility $\frac{1}{T}\sum_{t=1}^{T} U_i^t$, with $U_i^t$ defined in~\cref{eq:client_utility}. Specifically, in round $t$, it selects a noise scale $\sigma_i^t$ from the noise action set $\Sigma=\{\sigma^{(1)},\cdots,\sigma^{(\mathbb{K})}\}$. The client maintains an empirical estimator $\hat{\mu}_k$ of the expected per-round utility when selecting $\sigma^{(k)}$, and a selection count $n_k$. In round $t$, it computes the upper-confidence-bound (UCB) index $I_k^t$:
\begin{equation}
\label{eq:UCB}
    I_k^t \;=\; \hat{\mu}_k \;+\; \kappa \sqrt{\frac{2\ln t}{\max(1,n_k)}},
\end{equation}
where $\hat{\mu}_k$ is the exploitation term, and the second term is the exploration bonus (line~\ref{line:compute}). The client then selects the action with the largest $I_k^t$, sets $\sigma_i^t=\sigma^{(k)}$, and uploads the locally perturbed update (lines~\ref{line:select}--~\ref{line:upload}). After receiving the new global model $W^{t}$, it evaluates $U_i^t$ (lines~\ref{line:observe}--~\ref{line:utility}) and then updates the chosen action’s statistics for next round: $n_k \leftarrow n_k+1$ and $\hat{\mu}_{k} \leftarrow (1-\beta)\,\hat{\mu}_{k} + \beta\, U_i^t$, where the decay parameter $\beta$ keeps $\hat{\mu}_{k}$ responsive to slow drift due to continually arriving data (line~\ref{line:update}). As the confidence bonus shrinks, the selection stabilizes on the utility-maximizing noise scale. The procedure is summarized in~\Cref{alg:client_best_response}.

Within WinFLoRA, lower noise typically yields larger performance gains by increasing aggregation weights. Thus, performance-oriented clients obtain larger $\hat\mu_k$ at low noise, and UCB concentrates on low-noise actions. Conversely, when strong privacy preferences make a client’s utility larger at a higher noise level, the corresponding high-noise action attains a larger $\hat\mu_k$, and the client's noise adaptation converges to that level. In short, the combination of UCB and WinFLoRA motivates individual noise adaptation toward each client’s utility-maximizing level.

\subsection{Theoretical Analysis} 
\label{subsec:theory}
We model the interaction between the server and clients as a \textbf{stochastic aggregative Markov game (SAMG)}. In each round, clients choose noise levels based on the current global state, and the server aggregates the updates via inverse-noise weighting. Clients' utilities and state transitions depend on client actions through the aggregate statistic. Due to space limitations, the SAMG definition, the stationary Markov equilibrium (SME) of clients, and its existence proof are detailed in~\Cref{appendix: theoretical_analysis}.

\section{Experimental Results}
\begin{table*}[htbp]
\centering
\caption{\textbf{Global and individual performance across various models and datasets. The values of $\mathcal{A}_G$ and $\overline{\mathcal{A}_L}$} are presented in percentage format (\%). The best results are highlighted in light blue and bold.} 
\label{tab:overall_table}
\footnotesize
\setlength{\tabcolsep}{4pt}
\renewcommand{\arraystretch}{1.05}
\begin{tabular}{l l cccc|cccc|cccc}
\toprule
\multirow{3}{*}{Model} & \multirow{3}{*}{Method} &
\multicolumn{4}{c}{AGNews} &
\multicolumn{4}{c}{DBpedia} &
\multicolumn{4}{c}{20Newsgroups} \\
\cmidrule(lr){3-6}\cmidrule(lr){7-10}\cmidrule(lr){11-14}
& & $\mathcal{A}_G$  & $\overline{\mathcal{U}_L}$ & $\overline{\mathcal{A}_L}$  & $\overline{\mathcal{N}}$
  & $\mathcal{A}_G$  & $\overline{\mathcal{U}_L}$ & $\overline{\mathcal{A}_L}$  & $\overline{\mathcal{N}}$
  & $\mathcal{A}_G$  & $\overline{\mathcal{U}_L}$ & $\overline{\mathcal{A}_L}$  & $\overline{\mathcal{N}}$ \\
\midrule
\multirow{7}{*}{TinyLlama}
& FedIT       & 24.55 & 0.421 & 33.32 & 0.574 & 42.56 & 0.488 & 41.94 & 0.605 & 31.83 & 0.456 & 33.45 & 0.759 \\
& Flora       & 22.63 & 0.289 & 0.214  & 0.574 & 55.00 & 0.555 & 52.72 & 0.605 & 27.32 & 0.420 & 30.00 & 0.759 \\
\cmidrule(lr){2-14}
& FedIT+INA   & 27.43 & 0.391 & 26.63 & \cellcolor{cyan!8!white}\textbf{0.675} & 29.54 & 0.424 & 30.38 & \cellcolor{cyan!8!white}\textbf{0.697} & 31.45 & 0.467 & 33.17 & 0.743 \\
& Flora+INA   & 26.84 & 0.396 & 27.35 & 0.616 & 26.65 & 0.409 & 28.86 & 0.691 & 25.02 & 0.402 & 25.12 & \cellcolor{cyan!8!white}\textbf{0.763} \\
& FedMT+INA   & 53.45 & 0.490 & 55.22 & 0.606 & 39.44 & 0.418 & 44.28 & 0.617 & 29.78 & 0.369 & 32.13 & 0.692    \\
\cmidrule(lr){2-14}
& \cellcolor{cyan!8!white}\textbf{WinFLoRA} & \cellcolor{cyan!8!white}\textbf{75.21} & \cellcolor{cyan!8!white}\textbf{0.739} & \cellcolor{cyan!8!white}\textbf{76.03} & 0.574 & \cellcolor{cyan!8!white}\textbf{62.44} & \cellcolor{cyan!8!white}\textbf{0.638} & \cellcolor{cyan!8!white}\textbf{62.85} & 0.605 & \cellcolor{cyan!8!white}\textbf{42.58} & \cellcolor{cyan!8!white}\textbf{0.551} & \cellcolor{cyan!8!white}\textbf{43.96} & 0.759 \\
\midrule
\multirow{7}{*}{GPT2-Large}
& FedIT       & 65.33 & 0.634 & 65.56 & 0.547 & 40.00 & 0.448 & 38.66 & 0.600 & 32.36 & 0.446 & 33.35 & 0.736 \\
& Flora       & 62.54 & 0.653 & 67.48 & 0.547 & 50.98 & 0.532 & 51.61 & 0.600 & 40.05 & 0.493 & 39.81 & 0.736 \\
\cmidrule(lr){2-14}
& FedIT+INA   & 58.97 & 0.626 & 59.94 & \cellcolor{cyan!8!white}\textbf{0.569} & 33.38 & 0.497 & 36.56 & 0.664 & 31.96 & 0.456 & 32.52 & 0.748 \\
& Flora+INA   & 52.75 & 0.521 & 53.71 & 0.560 & 31.62 & 0.459 & 32.58 & \cellcolor{cyan!8!white}\textbf{0.742} & 41.53 & 0.542 & 42.57 & \cellcolor{cyan!8!white}\textbf{0.786} \\
& FedMT+INA   & 56.68 & 0.582 & 55.21  & 0.509 & 35.40 & 0.500 & 40.14 & 0.583 & 40.25 & 0.481 & 39.12  & 0.635    \\
\cmidrule(lr){2-14}
& \cellcolor{cyan!8!white}\textbf{WinFLoRA}& \cellcolor{cyan!8!white}\textbf{79.76} & \cellcolor{cyan!8!white}\textbf{0.726} & \cellcolor{cyan!8!white}\textbf{76.02} & 0.547 & \cellcolor{cyan!8!white}\textbf{56.04} & \cellcolor{cyan!8!white}\textbf{0.611} & \cellcolor{cyan!8!white}\textbf{56.98} & 0.600 & \cellcolor{cyan!8!white}\textbf{44.68} & \cellcolor{cyan!8!white}\textbf{0.569} & \cellcolor{cyan!8!white}\textbf{48.06} & 0.736 \\
\bottomrule
\end{tabular}

\end{table*}

\label{sec:exp_results}
\subsection{Experimental Setting}
\noindent\textbf{Large language models and datasets.} We evaluate our approach on two widely used models, including \textbf{TinyLlama}~\cite{zhang2024tinyllama} and \textbf{GPT2-Large}~\cite{radford2019language}, across three benchmark datasets covering diverse text classification tasks, i.e., \textbf{AGNews}~\cite{zhang2015character}, \textbf{DBpedia}~\cite{zhang2015character}, and \textbf{20Newsgroups}~\cite{lang1995newsweeder}. Details can be found in ~\Cref{appendix:models_detail}.

\noindent \textbf{Benchmarks.} We compare WinFLoRA with five benchmarks that span federated fine tuning and incentive mechanisms. Details are provided in~\Cref{appendix:benchmarks}.

\noindent \textbf{Evaluation metrics.} Four metrics are employed to assess the global model performance and the client-side utility, covering both local performance and privacy protection.
\begin{itemize}[leftmargin=*,noitemsep]
    \item \textbf{Global accuracy $\mathcal{A}_G$.} $\mathcal{A}_G$ measures the overall task performance of the global model, calculated by $\frac{1}{T}\sum_{t=1}^T \mathcal{A}_g^t$, where $\mathcal{A}_g^t$ is the evaluation accuracy of the global model in round $t$. 
    \item \textbf{Average utility $\overline{\mathcal{U}_L}$.} It measures clients' utilities by combining local performance and privacy preservation, i.e., $\overline{\mathcal{U}_L}= \frac{1}{NT}\sum_{t=1}^{T}$ $\sum_{i=1}^{N} \mathcal{U}_i^{\,t}$, where $\mathcal{U}_i^{\,t}=\frac{\mathcal{A}_i^{\,t}+\gamma_i \mathcal{P}_i^{\,t}}{1+\gamma_i}$. 
    \item \textbf{Average local accuracy $\overline{\mathcal{A}_L}$.} This measures how well the global model fits client specific data distributions on average. Defined as $\overline{\mathcal{A}_L}=\frac{1}{NT}\sum_{t=1}^T\sum_{i=1}^N \mathcal{A}_i^t$. 
    \item \textbf{Average noise scale $\overline{\mathcal{N}}$.} It indicates the scale of client side noise. Defined as $\overline{\mathcal{N}}=\frac{1}{TN}\sum_{t=1}^T\sum_{i=1}^N \sigma_i^t/\sigma_{\max}$. Higher values reflect stronger privacy protection on average.
\end{itemize}

\noindent\textbf{Experimental setup and implementation details.}
By default, datasets are split non-IID with a Dirichlet partition with $\alpha_{dir}=0.3$ over {N=10} clients, with 500 training samples per client, following FL/FFT common settings~\cite{koo2024towards,liu2024decentralized}. We run $T=20$ communication rounds of federated LoRA. Each client's privacy preference is sampled from $\gamma_i \sim \mathcal{N}(\mu,\,0.1^2)$ with $\mu=0.5$, and the noise upper bound is $\sigma_{\max}=0.1$.
Each client selects $\sigma_i^t=\sigma_k/\sigma_{max}$, where $\sigma_k \in \Sigma=\{0,\,0.1,\,0.5,\,1.0\}$. More details are provided in~\Cref{appendix:settings}

\subsection{Overall Performance}
\Cref{tab:overall_table} demonstrates that WinFLoRA outperforms existing benchmarks with large margins across multiple models and datasets.
As shown in~\Cref{tab:overall_table}, WinFLoRA achieves $\mathcal{A}_G$ of 75.21\% and 79.76\% on AGNews with TinyLlama and GPT2-Large, outperforming the benchmarks by 21.76-52.58\% and 14.43-27.01\%, respectively. Specially, WinFLoRA achieves $\overline{\mathcal{A}_L}$ of 76.03\% on TinyLlama and 76.02\% on GPT2-Large, exceeding the benchmarks by 20.81–54.63\% and 10.46–28.34\%. Similar advantages hold on DBpedia and 20NG. These results indicate that WinFLoRA consistently amplifies low-noise, high-quality updates and discounts high-noise contributions, thereby preventing severe degradation of the global model. When the $\overline{\mathcal{N}}$ matches the non-INA benchmarks of 0.574 on AGNews, WinFLoRA outperforms FedIT and Flora by 50.66\% and 52.58\% in $\mathcal{A}_G$, respectively. This is because average weighting ignores noise scale, fails to suppress high-noise updates and reduces global accuracy. Compared with INA benchmarks where clients adapt noise over rounds, WinFLoRA achieves slightly lower $\overline{\mathcal{N}}$ on AGNews and DBpedia by 3.1-14.9\% and a comparable level on 20NG, while $\mathcal{A}_G$ remains substantially higher. Although FedMT+INA encourages clients to trade privacy for accuracy and lowers $\overline{\mathcal{N}}$ on GPT2-Large, WinFLoRA attains 23.08\%, 20.64\% and 4.43\% higher $\mathcal{A}_G$ on AGNews, DBpedia, and 20NG, respectively. This is because FedMT+INA rewards relative noise reductions, but high-noise updates still retain substantial weight. These demonstrate that noise-aware weighting improves global accuracy while accommodating heterogeneous privacy preferences.

\noindent\textbf{Individual utility.} 
Average client utility is consistently higher with WinFLoRA, indicating that self-interested choices are effectively incentivized. With TinyLlama, the $\overline{\mathcal{U}_L}$ reaches 0.739, 0.638, and 0.551 on AGNews, DBpedia, and 20NG, respectively, representing 1.50-2.55$\times$ achieved by benchmarks. Higher $\overline{\mathcal{U}_L}$, alongside higher $\mathcal{A}_G$, indicates that heterogeneous clients obtain higher rewards under WinFLoRA. Lower-noise clients receive stronger, more positive feedback of greater influence, while privacy-prioritizing clients learn to stabilize at an optimal noise scale where their own utility peaks, driven by their privacy requirements.

\subsection{Ablation Study}
\noindent\textbf{Effectiveness of NWA.}
To evaluate the effectiveness of NWA, we compare the global performance and convergence with and without NWA. On TinyLlama, NWA brings significant advantages in global accuracy and convergence, i.e. $\mathcal{A}_g^T$ of 80.07\%, which is 54.63\% higher than without NWA. On GPT2-Large, it attains $\mathcal{A}_g^T$ of 84.33\%, 28.39\% higher than that without NWA. NWA also achieves more rapid and stable convergence, stabilizing by round 8, whereas it takes longer to stabilize without NWA, i.e., after round 16. On GPT2-large, with NWA achieves 2 rounds earlier stabilization than without NWA. In terms of $\overline{\mathcal{U}_L}$, NWA shows similar advantages. With NWA, $\overline{\mathcal{U}_L}$ increases steadily and stabilizes near 0.790 and 0.754 on TinyLlama and GPT2-Large, respectively, 0.406 and 0.204 higher than without NWA.
NAW down-weights noisier client updates and feeds this weighting signal back to the clients, effectively translating the marginal global degradation caused by added noise into an immediate reduction in the client's own performance reward. This reduces the strategic reactivity of each client to others’ actions, thereby promoting faster and more stable convergence.
\begin{figure}[h]
  \centering
    \begin{subfigure}[t]{\linewidth}
    \centering
    \includegraphics[width=0.5\linewidth]{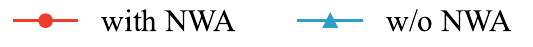}
  \end{subfigure}\hfill
  \begin{subfigure}[t]{0.48\linewidth}
    \centering
    \includegraphics[width=\linewidth]{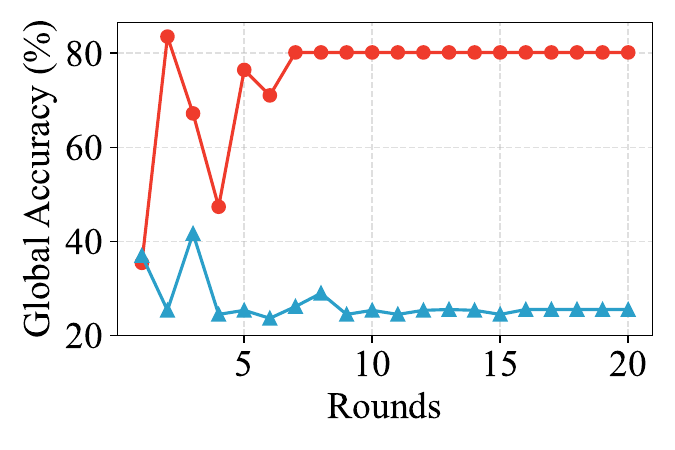}
    \subcaption{Global Accuracy (TinyLlama)}
    \label{subfig:conver_global_tiny}
  \end{subfigure}\hfill
  \begin{subfigure}[t]{0.48\linewidth}
    \centering
    \includegraphics[width=\linewidth]{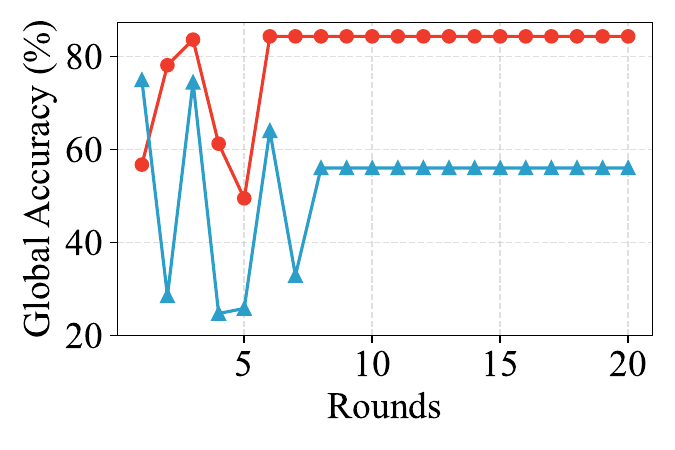}
    \subcaption{Global Accuracy (GPT2-Large)}
  \end{subfigure}
  
  \begin{subfigure}[t]{0.48\linewidth}
    \centering
    \includegraphics[width=\linewidth]{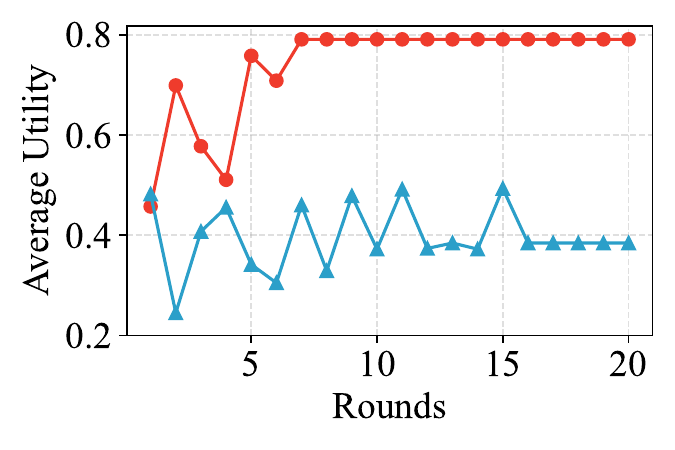}
    \subcaption{Average Utility (TinyLlama)}
    \label{subfig:conver_util_tiny}
  \end{subfigure}\hfill
  \begin{subfigure}[t]{0.48\linewidth}
    \centering
    \includegraphics[width=\linewidth]{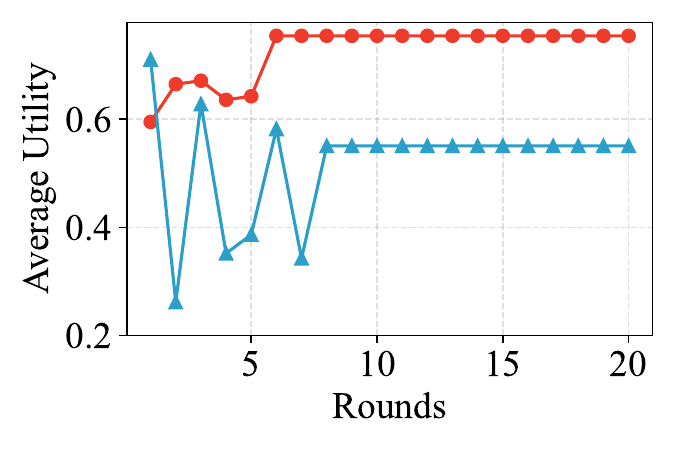}
    \subcaption{Average Utility (GPT2-Large)}
  \end{subfigure}

  \caption{Performance on AGNews with and without NWA. “Without NWA” refers to a no-incentive baseline with average aggregation. Clients adapt noise with INA in both settings.}
  \label{fig:incentive_effectiveness}
\end{figure}

\noindent\textbf{Adaptability to privacy heterogeneity.} 
To evaluate performance under heterogeneous privacy preferences, we sample $\gamma_i \sim \mathcal{N}(\mu, 0.1^2)$ with $\mu \in \{0.3, 0.4, 0.5, 0.6, 0.7\}$, where a larger $\mu$ indicates the client with a higher privacy preference, i.e., privacy-oriented. As shown in ~\Cref{fig:privacy_preference}, when performance-oriented clients dominate with $\mu=0.3$, the system achieves a high global accuracy $\mathcal{A}_G$ of $83.52\%$ on TinyLlama and $79.66\%$ on GPT2-Large, together with a high average utility $\bar{\mathcal{U}}_L$ of $0.801$ and $0.789$, respectively. As $\mu$ increases, the equilibrium shifts toward stronger privacy and the mean noise $\overline{\mathcal{N}}$ rises from $0.481$ to $0.668$ on TinyLlama and from $0.552$ to $0.610$ on GPT2-Large, with a decline in $\mathcal{A}_G$. Even at $\mu=0.7$, where $\overline{\mathcal{N}}=0.688$ on TinyLlama and $0.610$ on GPT2-Large, the system still achieves $\mathcal{A}_G$ of $30.66\%$ and $50.8\%$, respectively.

\noindent\textbf{Impact of noise action scales.} 
We evaluate the impact of the granularity of clients’ noise action sets on performance and convergence with coarse (the default noise action set), moderate, and fine in~\Cref{subfig:conver_global_tiny}. Compared to the coarse set, the moderate and fine sets raise $\mathcal{A}_G$ to 83.08\% and 84.58\%, outperforming the coarse set by 7.87\% and 9.37\%, respectively, as shown in~\Cref{subfig:strategy_acc}. Finer grids enable the selection of $\sigma_i^t$ closer to the optimal utility-maximizing noise scale. Accordingly, as shown in~\Cref{subfig:strategy_util}, $\overline{\mathcal{U}_L}$ increases to 0.777 and 0.783 with the moderate and fine set, respectively, compared with the coarse set shown in~\Cref{subfig:conver_util_tiny}. The trade-off is slower stabilization. The coarse set stabilizes fastest, the moderate set stabilizes at 9 rounds (+2 rounds), and the fine set at 21 rounds (+14 rounds). Finer granularity shrinks the effective step size, so more exploration is required before clients reach a high-utility steady state. In practice, one should tune the actions set granularity to match system requirements, prioritizing either faster convergence or higher global accuracy.

\begin{figure}[h]
  \centering
  \begin{subfigure}[t]{\linewidth}
    \centering
    \includegraphics[width=0.5\linewidth]{figs/legend_1.pdf}
  \end{subfigure}\hfill
  \begin{subfigure}[t]{0.48\linewidth}
    \centering
    \includegraphics[width=\linewidth]{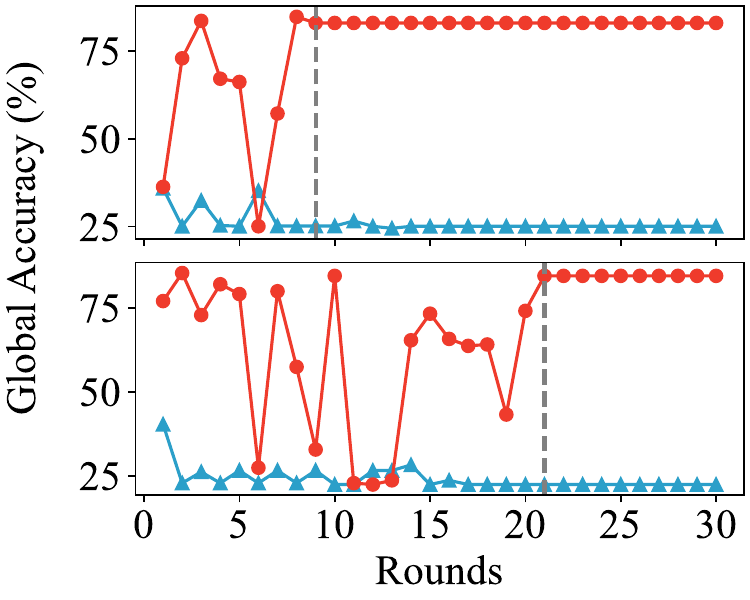}
    \subcaption{Global Accuracy.}
    \label{subfig:strategy_acc}
  \end{subfigure}\hfill
  \begin{subfigure}[t]{0.48\linewidth}
    \centering
    \includegraphics[width=\linewidth]{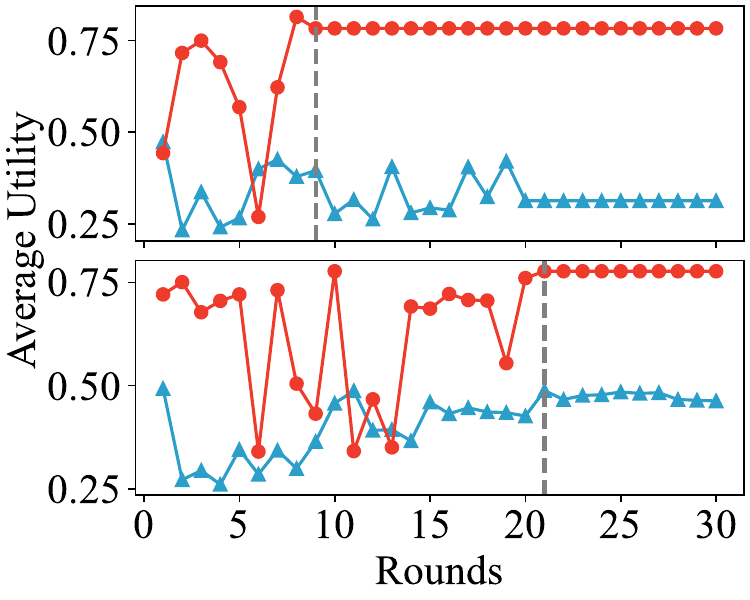}
    \subcaption{Average utility.}
    \label{subfig:strategy_util}
  \end{subfigure}
  \caption{Performance vs. noise action scales. $\mathcal{A}_G$ and $\overline{\mathcal{U}_L}$ on the AGNews dataset with WinFLoRA under two optional noise scale sets, compared against a w/o NWA benchmark. Top: moderate set $\Sigma_{mod}=\{0.0,0.2,\cdots,1.0\}$. Bottom: fine set $\Sigma_{fine}=\{0.0,0.1,\cdots,1.0\}$.}
  \label{fig:strategy}
\end{figure}

\noindent\textbf{Adaptability to data heterogeneity.} 
As data heterogeneity weakens with Dirichlet concentration, we evaluate the impact of $\alpha_{dir}$ varying from 0.1 to 0.5 on WinFLoRA in~\Cref{fig:noniid}. In general, WinFLoRA enhances privacy protection without compromising performance. 
As shown in~\Cref{fig:noniid}, $\overline{\mathcal{N}}$ increases from 0.485 to a peak of 0.647 on TinyLlama and from 0.514 to 0.880 on GPT2-Large, accompanied by improvements in $\mathcal{A}_G$ by 25.60\% and 48.32\%, and in $\overline{\mathcal{U}_L}$ by 0.185 and 0.259, respectively. As client update directions become increasingly aligned, low-noise clients carry most of the useful signal in the aggregate, whereas high-noise clients preserve privacy with negligible impact on the global model. Near the IID regime at $\alpha_{dir}=0.5$ on TinyLlama, 
the drop of $\overline{\mathcal{N}}$ to 0.583 reflects diminishing privacy returns from added noise and sharper performance trade-offs. Consequently, lower noise becomes preferable, trading off for the increase in $\mathcal{A}_G$ to 79.75\% and in $\overline{\mathcal{U}_L}$ to 0.765. By contrast, GPT2-Large offers greater capacity and redundancy to tolerate higher $\overline{\mathcal{N}}$ while improving both $\mathcal{A}_G$ and $\overline{\mathcal{U}_L}$.
\begin{figure}[ht]
  \centering
  \begin{subfigure}[t]{\linewidth}
    \centering
    \includegraphics[width=0.9\linewidth]{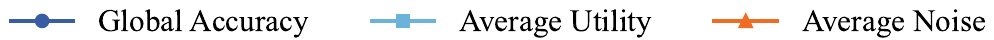}
  \end{subfigure}\hfill
  \begin{subfigure}[t]{0.49\linewidth}
    \centering
    \includegraphics[width=\linewidth]{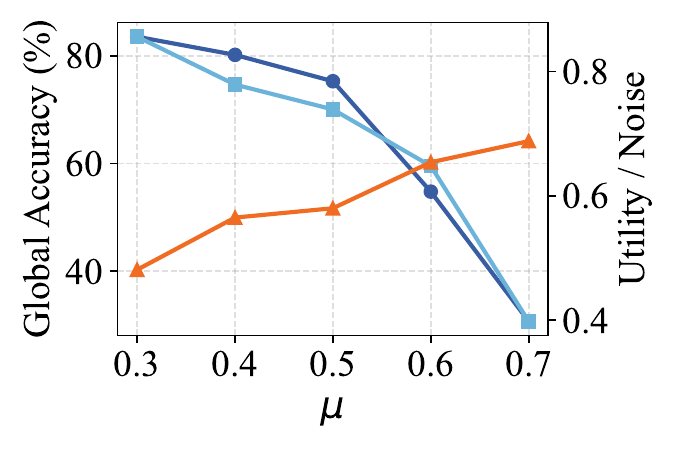}
    \subcaption{TinyLlama.}
  \end{subfigure}\hfill
  \begin{subfigure}[t]{0.49\linewidth}
    \centering
    \includegraphics[width=\linewidth]{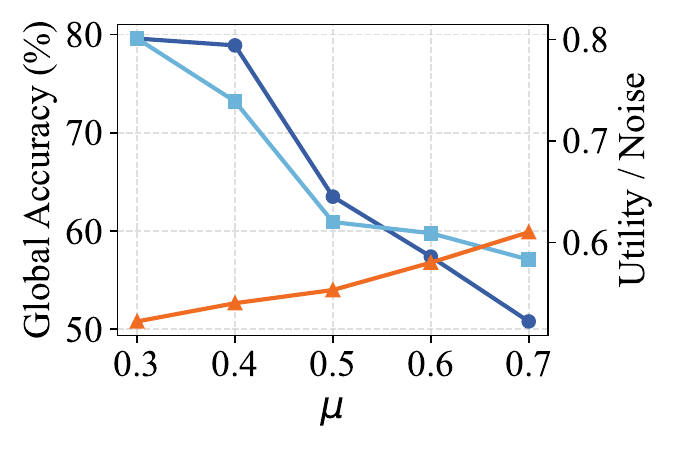}
    \subcaption{GPT2-Large.}
  \end{subfigure}
  \caption{Performance vs. privacy preference. $\mathcal{A}_G$, $\overline{\mathcal{U}_L}$ and $\overline{\mathcal{N}}$ across client's privacy preference, with the Gaussian mean $\mu$ for sampling $\gamma_i$ increasing from 0.3 to 0.7.}
  \label{fig:privacy_preference}
\end{figure}

\begin{figure}[h]
  \centering
    \begin{subfigure}[t]{\linewidth}
    \centering
    \includegraphics[width=\linewidth]{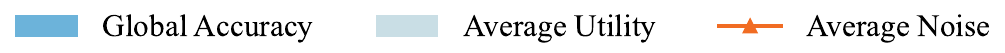}
  \end{subfigure}\hfill
  \begin{subfigure}[t]{0.49\linewidth}
    \centering
    \includegraphics[width=\linewidth]{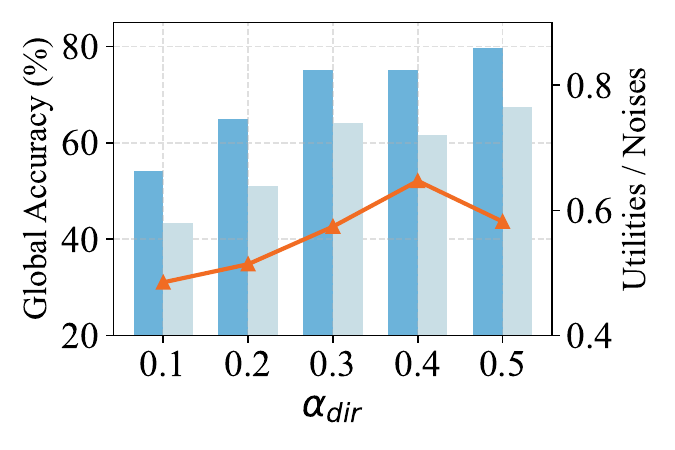}
    \subcaption{TinyLlama.}
  \end{subfigure}\hfill
  \begin{subfigure}[t]{0.49\linewidth}
    \centering
    \includegraphics[width=\linewidth]{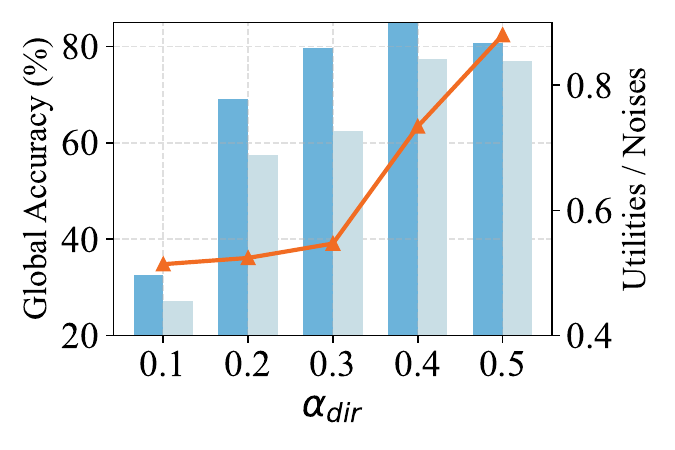}
    \subcaption{GPT2-Large.}
  \end{subfigure}
  \caption{Performance vs. $\alpha_{dir}$. $\mathcal{A}_G$, $\overline{\mathcal{U}_L}$ and  $\overline{\mathcal{N}}$ across non-iid settings with $\alpha_{dir}$ increasing from 0.1 to 0.5.}
  \label{fig:noniid}
\end{figure}

\noindent\textbf{Impact of client numbers.} 
We evaluate the performance under various system scales with N increasing from 6 to 14. As the number of clients $N$ increases, $\mathcal{A}_G$ increases from 49.58\% to 78.55\% on TinyLlama, and from 55.45\% to 85.58\% on GPT2-Large. Meanwhile, the $\overline{\mathcal{N}}$ decreases from 0.603 to 0.510 and on TinyLlama and 0.713 to 0.494 on GPT2-Large, as shown in~\Cref{fig:client_number}. The reason is that the growing number of clients intensifies competition for influence on the global model, as aggregation weights are normalized to sum to 1. Under noise-aware aggregation, a practical way to gain influence is to submit a cleaner (lower-noise) update. Performance-oriented clients, therefore, reduce noise to capture higher marginal performance gains. When global accuracy begins to plateau, the competition is amplified, resulting in a step-down in $\overline{\mathcal{N}}$ at $N=14$ on TinyLlama and $N=8$ on GPT2-Large.
\begin{figure}[ht]
  \centering
    \begin{subfigure}[t]{\linewidth}
    \centering
    \includegraphics[width=0.9\linewidth]{figs/legend_3.pdf}
  \end{subfigure}\hfill
  \begin{subfigure}[t]{0.49\linewidth}
    \centering
    \includegraphics[width=\linewidth]{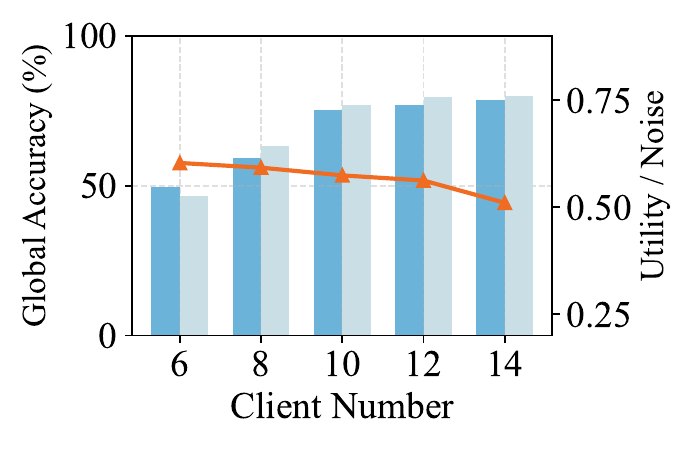}
    \subcaption{TinyLlama.}
  \end{subfigure}\hfill
  \begin{subfigure}[t]{0.49\linewidth}
    \centering
    \includegraphics[width=\linewidth]{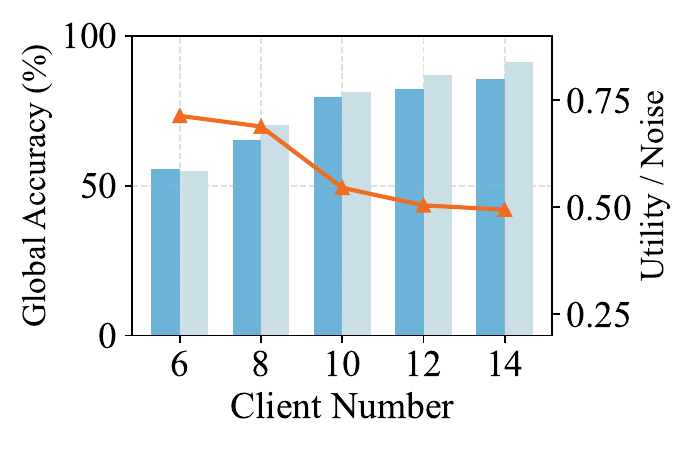}
    \subcaption{GPT2-Large.}
  \end{subfigure}
  \caption{Performance vs. number of clients. $\mathcal{A}_G$, $\overline{\mathcal{U}_L}$, and $\overline{\mathcal{N}}$ across system scales with the number of clients increasing from 6 to 14.}
  \label{fig:client_number}
\end{figure}

\noindent\textbf{Individual improvement from FFT.}
Clients are motivated to participate in WinFLoRA to obtain higher individual utility, compared to fine-tuning in isolation, with or without local noise. As shown in~\Cref{fig:standalone}, on TinyLlama, the average accuracy improves by 14.38\%-22.26\%, $1.24\times$-$1.67\times$ over standalone and by 33.11\%-51.42\%, $3.16\times$-$5.34\times$ over Sta(N), respectively. The advantages of WinFLoRA show a similar pattern on GPT2-Large. This demonstrates the strong motivation for clients to participate in WinFLoRA as they benefit from it. This improvement stems from two key factors: aggregation enables clients to benefit from knowledge in domains they do not observe locally, and WinFLoRA further enhances this by down-weighting noisy ones, thus avoiding the performance degradation.

\begin{figure}[ht]
  \centering
  \begin{subfigure}[t]{0.32\linewidth}
    \centering
    \includegraphics[width=\linewidth]{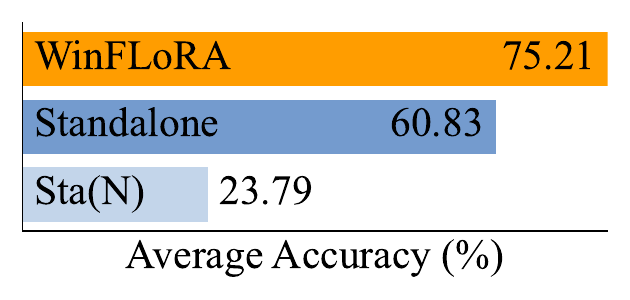}
    \subcaption{TinyLlama(AGNews).}
  \end{subfigure}
    \begin{subfigure}[t]{0.32\linewidth}
    \centering
    \includegraphics[width=\linewidth]{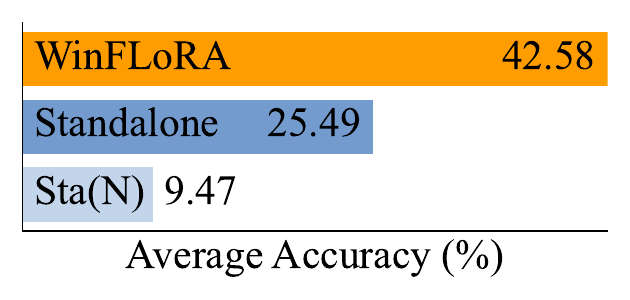}
    \subcaption{TinyLlama(DBpedia).}
  \end{subfigure}
  \begin{subfigure}[t]{0.32\linewidth}
    \centering
    \includegraphics[width=\linewidth]{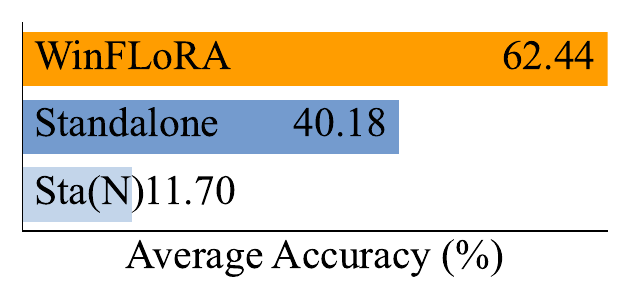}
    \subcaption{TinyLlama(20NG).}
  \end{subfigure}
    \begin{subfigure}[t]{0.32\linewidth}
    \centering
    \includegraphics[width=\linewidth]{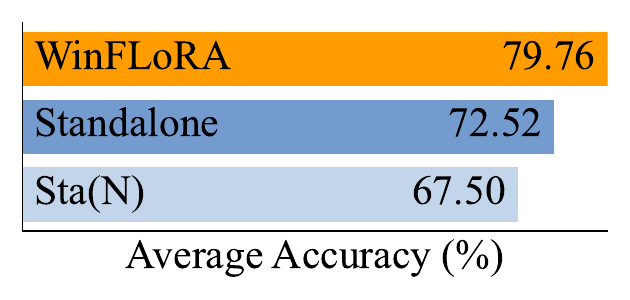}
    \subcaption{GPT2-L(AGNews).}
  \end{subfigure}
  \begin{subfigure}[t]{0.32\linewidth}
    \centering
    \includegraphics[width=\linewidth]{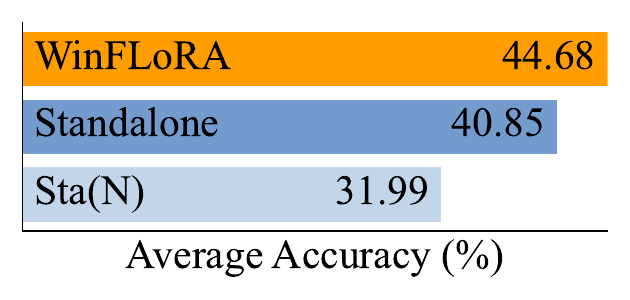}
    \subcaption{GPT2-L(DBpedia).}
  \end{subfigure}
    \begin{subfigure}[t]{0.32\linewidth}
    \centering
    \includegraphics[width=\linewidth]{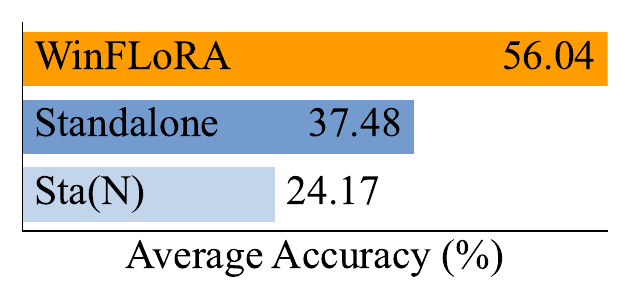}
    \subcaption{GPT2-L(20NG).}
  \end{subfigure}
  \caption{Average accuracy under standalone FT. Clients fine-tune locally without aggregation, evaluated on all classes of the dataset. Results are shown for clean LoRA (Standalone) and noisy LoRA (Sta(N)). }
  \label{fig:standalone}
\end{figure}
\noindent\textbf{Impact of noise estimation.} 
To evaluate the impact of the noise estimation, we compare the true injected noise scales with the estimated scales from two perspectives: each client's noise scale trajectory over communication rounds, and the noise distribution across all clients at a fixed round (e.g., $t=1$). As illustrated in~\Cref{fig:estimate_time}, the average gap between estimated and true noise is below 1\%, indicating that estimation does not distort incentives. Over rounds, the estimated noise closely tracks the true noise for each client across all datasets, so the incentive has a negligible effect on their utility-maximizing behavior. Within each aggregation round, the cross-client ranking of estimated noise also matches the true ranking, ensuring that cleaner updates are consistently up-weighted, which is the key property for effective incentives and weighted aggregation.

\begin{figure}[ht]
  \centering
    \begin{subfigure}[t]{\linewidth}
    \centering
    \includegraphics[width=0.4\linewidth]{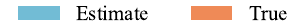}
  \end{subfigure}\hfill
  \begin{subfigure}[t]{0.33\linewidth}
    \centering
    \includegraphics[width=\linewidth]{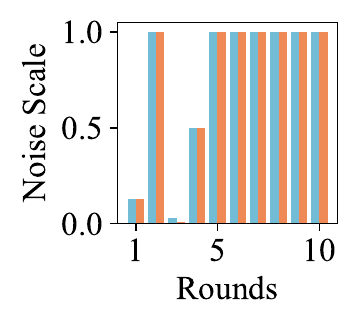}
    \subcaption{AGNews.}
  \end{subfigure}\hfill
  \begin{subfigure}[t]{0.33\linewidth}
    \centering
    \includegraphics[width=\linewidth]{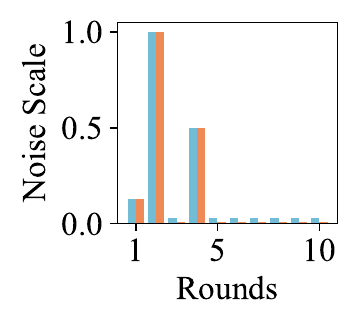}
    \subcaption{DBpedia.}
  \end{subfigure}
  \begin{subfigure}[t]{0.33\linewidth}
    \centering
    \includegraphics[width=\linewidth]{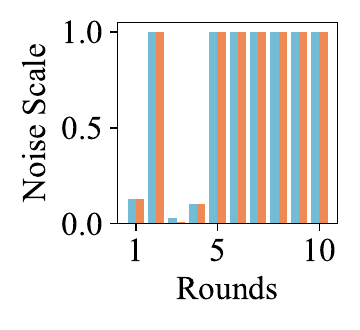}
    \subcaption{20NG.}
  \end{subfigure}
  \begin{subfigure}[t]{0.33\linewidth}
    \centering
    \includegraphics[width=\linewidth]{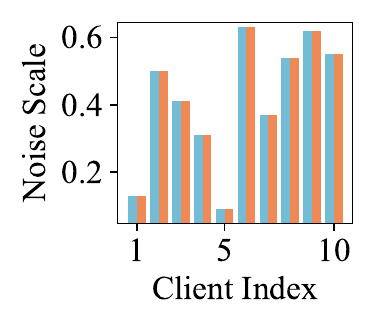}
    \subcaption{AGNews.}
  \end{subfigure}\hfill
  \begin{subfigure}[t]{0.33\linewidth}
    \centering
    \includegraphics[width=\linewidth]{figs/estimate_cn_tiny.pdf}
    \subcaption{DBpedia.}
  \end{subfigure}
  \begin{subfigure}[t]{0.33\linewidth}
    \centering
    \includegraphics[width=\linewidth]{figs/estimate_cn_tiny.pdf}
    \subcaption{20NG.}
  \end{subfigure}
  \caption{Estimated noise vs. true noise. Top: client $c_1$'s noise scale over rounds. Bottom: per-client noise scale in round $t_1$.}
  \label{fig:estimate_time}
\end{figure}
To evaluate the tolerance to estimation bias, we apply a perturbation to the estimated noise $\hat{\sigma}_i^t$ by replacing $\hat{\sigma}_i^t$ with $\tilde{\sigma}_i^t=(\hat{\sigma}_i^t)^\rho$ and recompute aggregation weights from $\tilde{\sigma}_i^t$. As shown in~\cref{fig:estimate_bias}, the $\mathcal{A}_g$ achieves 89.76-94.91\%, 88.34-89.26\%, and  92.94-95.52\% of the no-bias baseline ($\rho=1$), on AGnews, DBpedia and 20NG, respectively. The mild degradation indicates that the relative ranking of clients' noise scales is sufficient for NWA and the overall performance remains strong within WinFLoRA despite estimation bias.

\begin{figure}[ht]
  \centering
    \begin{subfigure}[t]{\linewidth}
    \centering
    \includegraphics[width=0.5\linewidth]{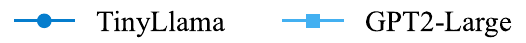}
  \end{subfigure}\hfill
  \begin{subfigure}[t]{0.33\linewidth}
    \centering
    \includegraphics[width=\linewidth]{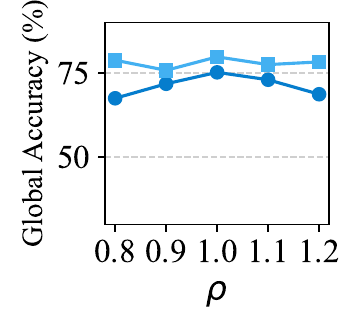}
    \subcaption{AGNews.}
  \end{subfigure}\hfill
  \begin{subfigure}[t]{0.33\linewidth}
    \centering
    \includegraphics[width=\linewidth]{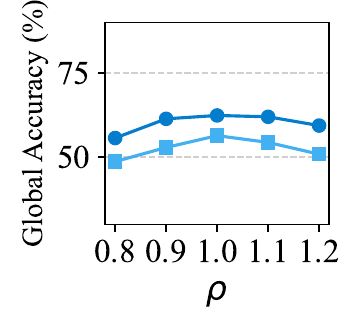}
    \subcaption{DBpedia.}
  \end{subfigure}
  \begin{subfigure}[t]{0.33\linewidth}
    \centering
    \includegraphics[width=\linewidth]{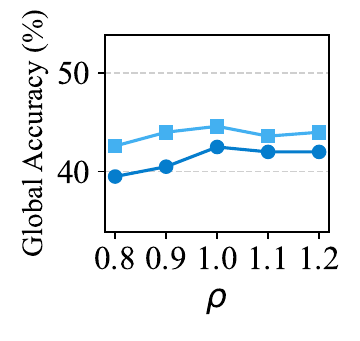}
    \subcaption{20NG.}
  \end{subfigure}
  \caption{Global accuracy vs. bias factor $\rho$. $\mathcal{A}_G$ across $\rho \in \{0.8,0.9,1.0,1.1,1.2\}$ with $\rho=1$ denoting no estimation bias.}
  \label{fig:estimate_bias}
\end{figure}

\section{Conclusion}
In this paper, we address the fundamental trade-off between client-side privacy preservation and global model performance in federated fine-tuning under privacy heterogeneity. We propose WinFLoRA, a noise-aware aggregation framework that assigns higher influence to lower-noise updates while allowing privacy-prioritizing clients to retain strong privacy guarantees. By leveraging inverse-noise weighting as an incentive mechanism, WinFLoRA improves global model accuracy and aligns each client’s individual utility with the overall system objective. As a result, clients are encouraged to contribute higher-quality updates without compromising their privacy preferences. Overall, WinFLoRA achieves substantial gains in both global accuracy and average client utility in privacy-heterogeneous federated LoRA settings.

\section*{Acknowledgment}
This work is supported by the Australian Research Council Discovery Project DP220100215. 


\bibliographystyle{ACM-Reference-Format}
\balance
\bibliography{main}


\begin{thebibliography}{57}


\ifx \showCODEN    \undefined \def \showCODEN     #1{\unskip}     \fi
\ifx \showISBNx    \undefined \def \showISBNx     #1{\unskip}     \fi
\ifx \showISBNxiii \undefined \def \showISBNxiii  #1{\unskip}     \fi
\ifx \showISSN     \undefined \def \showISSN      #1{\unskip}     \fi
\ifx \showLCCN     \undefined \def \showLCCN      #1{\unskip}     \fi
\ifx \shownote     \undefined \def \shownote      #1{#1}          \fi
\ifx \showarticletitle \undefined \def \showarticletitle #1{#1}   \fi
\ifx \showURL      \undefined \def \showURL       {\relax}        \fi
\providecommand\bibfield[2]{#2}
\providecommand\bibinfo[2]{#2}
\providecommand\natexlab[1]{#1}
\providecommand\showeprint[2][]{arXiv:#2}

\bibitem[Achiam et~al\mbox{.}(2023)]%
        {achiam2023gpt}
\bibfield{author}{\bibinfo{person}{Josh Achiam}, \bibinfo{person}{Steven Adler}, \bibinfo{person}{Sandhini Agarwal}, \bibinfo{person}{Lama Ahmad}, \bibinfo{person}{Ilge Akkaya}, \bibinfo{person}{Florencia~Leoni Aleman}, \bibinfo{person}{Diogo Almeida}, \bibinfo{person}{Janko Altenschmidt}, \bibinfo{person}{Sam Altman}, \bibinfo{person}{Shyamal Anadkat}, {et~al\mbox{.}}} \bibinfo{year}{2023}\natexlab{}.
\newblock \showarticletitle{Gpt-4 technical report}.
\newblock \bibinfo{journal}{\emph{arXiv preprint arXiv:2303.08774}} (\bibinfo{year}{2023}).
\newblock


\bibitem[Babakniya et~al\mbox{.}(2023)]%
        {babakniya2023slora}
\bibfield{author}{\bibinfo{person}{Sara Babakniya}, \bibinfo{person}{Ahmed~Roushdy Elkordy}, \bibinfo{person}{Yahya~H Ezzeldin}, \bibinfo{person}{Qingfeng Liu}, \bibinfo{person}{Kee-Bong Song}, \bibinfo{person}{Mostafa El-Khamy}, {and} \bibinfo{person}{Salman Avestimehr}.} \bibinfo{year}{2023}\natexlab{}.
\newblock \showarticletitle{Slora: Federated parameter efficient fine-tuning of language models}.
\newblock \bibinfo{journal}{\emph{arXiv preprint arXiv:2308.06522}} (\bibinfo{year}{2023}).
\newblock


\bibitem[Birrell et~al\mbox{.}(2024)]%
        {birrell2024differentially}
\bibfield{author}{\bibinfo{person}{Jeremiah Birrell}, \bibinfo{person}{Reza Ebrahimi}, \bibinfo{person}{Rouzbeh Behnia}, {and} \bibinfo{person}{Jason Pacheco}.} \bibinfo{year}{2024}\natexlab{}.
\newblock \showarticletitle{Differentially private stochastic gradient descent with fixed-size minibatches: Tighter RDP guarantees with or without replacement}.
\newblock \bibinfo{journal}{\emph{Advances in Neural Information Processing Systems}}  \bibinfo{volume}{37} (\bibinfo{year}{2024}), \bibinfo{pages}{11087--11131}.
\newblock


\bibitem[Cai et~al\mbox{.}(2023)]%
        {cai2023efficient}
\bibfield{author}{\bibinfo{person}{Dongqi Cai}, \bibinfo{person}{Yaozong Wu}, \bibinfo{person}{Shangguang Wang}, \bibinfo{person}{Felix~Xiaozhu Lin}, {and} \bibinfo{person}{Mengwei Xu}.} \bibinfo{year}{2023}\natexlab{}.
\newblock \showarticletitle{Efficient federated learning for modern nlp}. In \bibinfo{booktitle}{\emph{Proceedings of the 29th annual international conference on mobile computing and networking}}. \bibinfo{pages}{1--16}.
\newblock


\bibitem[Chen et~al\mbox{.}(2024a)]%
        {chen2024fairreward}
\bibfield{author}{\bibinfo{person}{Guorong Chen}, \bibinfo{person}{Chao Li}, \bibinfo{person}{Wei Wang}, \bibinfo{person}{Li Duan}, \bibinfo{person}{Bin Wang}, \bibinfo{person}{Zhen Han}, {and} \bibinfo{person}{Xiangliang Zhang}.} \bibinfo{year}{2024}\natexlab{a}.
\newblock \showarticletitle{Fairreward: Towards fair reward distribution using equity theory in blockchain-based federated learning}.
\newblock \bibinfo{journal}{\emph{IEEE Transactions on Dependable and Secure Computing}} (\bibinfo{year}{2024}).
\newblock


\bibitem[Chen et~al\mbox{.}(2024c)]%
        {chen2024feddat}
\bibfield{author}{\bibinfo{person}{Haokun Chen}, \bibinfo{person}{Yao Zhang}, \bibinfo{person}{Denis Krompass}, \bibinfo{person}{Jindong Gu}, {and} \bibinfo{person}{Volker Tresp}.} \bibinfo{year}{2024}\natexlab{c}.
\newblock \showarticletitle{Feddat: An approach for foundation model finetuning in multi-modal heterogeneous federated learning}. In \bibinfo{booktitle}{\emph{Proceedings of the AAAI Conference on Artificial Intelligence}}, Vol.~\bibinfo{volume}{38}. \bibinfo{pages}{11285--11293}.
\newblock


\bibitem[Chen et~al\mbox{.}(2024b)]%
        {chen2024rbla}
\bibfield{author}{\bibinfo{person}{Shuaijun Chen}, \bibinfo{person}{Omid Tavallaie}, \bibinfo{person}{Niousha Nazemi}, {and} \bibinfo{person}{Albert~Y Zomaya}.} \bibinfo{year}{2024}\natexlab{b}.
\newblock \showarticletitle{Rbla: Rank-based-lora-aggregation for fine-tuning heterogeneous models in flaas}. In \bibinfo{booktitle}{\emph{International Conference on Web Services}}. Springer, \bibinfo{pages}{47--62}.
\newblock


\bibitem[Chen et~al\mbox{.}(2024d)]%
        {chen2024fdfl}
\bibfield{author}{\bibinfo{person}{Zhe Chen}, \bibinfo{person}{Haiyan Zhang}, \bibinfo{person}{Xinghua Li}, \bibinfo{person}{Yinbin Miao}, \bibinfo{person}{Xiaohan Zhang}, \bibinfo{person}{Man Zhang}, \bibinfo{person}{Siqi Ma}, {and} \bibinfo{person}{Robert~H Deng}.} \bibinfo{year}{2024}\natexlab{d}.
\newblock \showarticletitle{FDFL: Fair and discrepancy-aware incentive mechanism for federated learning}.
\newblock \bibinfo{journal}{\emph{IEEE Transactions on Information Forensics and Security}} (\bibinfo{year}{2024}).
\newblock


\bibitem[Ding et~al\mbox{.}(2025)]%
        {ding2025sulora}
\bibfield{author}{\bibinfo{person}{Chenhao Ding}, \bibinfo{person}{Jiangyang Li}, \bibinfo{person}{Songlin Dong}, \bibinfo{person}{Xinyuan Gao}, \bibinfo{person}{Yuhang He}, {and} \bibinfo{person}{Yihong Gong}.} \bibinfo{year}{2025}\natexlab{}.
\newblock \showarticletitle{SuLoRA: Subspace Low-Rank Adaptation for Parameter-Efficient Fine-Tuning}. In \bibinfo{booktitle}{\emph{Findings of the Association for Computational Linguistics: ACL 2025}}. \bibinfo{pages}{5334--5349}.
\newblock


\bibitem[Ding et~al\mbox{.}(2022)]%
        {ding2022delta}
\bibfield{author}{\bibinfo{person}{Ning Ding}, \bibinfo{person}{Yujia Qin}, \bibinfo{person}{Guang Yang}, \bibinfo{person}{Fuchao Wei}, \bibinfo{person}{Zonghan Yang}, \bibinfo{person}{Yusheng Su}, \bibinfo{person}{Shengding Hu}, \bibinfo{person}{Yulin Chen}, \bibinfo{person}{Chi-Min Chan}, \bibinfo{person}{Weize Chen}, {et~al\mbox{.}}} \bibinfo{year}{2022}\natexlab{}.
\newblock \showarticletitle{Delta tuning: A comprehensive study of parameter efficient methods for pre-trained language models}.
\newblock \bibinfo{journal}{\emph{arXiv preprint arXiv:2203.06904}} (\bibinfo{year}{2022}).
\newblock


\bibitem[Dodge et~al\mbox{.}(2020)]%
        {dodge2020fine}
\bibfield{author}{\bibinfo{person}{Jesse Dodge}, \bibinfo{person}{Gabriel Ilharco}, \bibinfo{person}{Roy Schwartz}, \bibinfo{person}{Ali Farhadi}, \bibinfo{person}{Hannaneh Hajishirzi}, {and} \bibinfo{person}{Noah Smith}.} \bibinfo{year}{2020}\natexlab{}.
\newblock \showarticletitle{Fine-tuning pretrained language models: Weight initializations, data orders, and early stopping}.
\newblock \bibinfo{journal}{\emph{arXiv preprint arXiv:2002.06305}} (\bibinfo{year}{2020}).
\newblock


\bibitem[Dwork et~al\mbox{.}(2006)]%
        {dwork2006calibrating}
\bibfield{author}{\bibinfo{person}{Cynthia Dwork}, \bibinfo{person}{Frank McSherry}, \bibinfo{person}{Kobbi Nissim}, {and} \bibinfo{person}{Adam Smith}.} \bibinfo{year}{2006}\natexlab{}.
\newblock \showarticletitle{Calibrating noise to sensitivity in private data analysis}. In \bibinfo{booktitle}{\emph{Theory of cryptography conference}}. Springer, \bibinfo{pages}{265--284}.
\newblock


\bibitem[Fu et~al\mbox{.}(2023)]%
        {fu2023practical}
\bibfield{author}{\bibinfo{person}{Wenjie Fu}, \bibinfo{person}{Huandong Wang}, \bibinfo{person}{Chen Gao}, \bibinfo{person}{Guanghua Liu}, \bibinfo{person}{Yong Li}, {and} \bibinfo{person}{Tao Jiang}.} \bibinfo{year}{2023}\natexlab{}.
\newblock \showarticletitle{Practical membership inference attacks against fine-tuned large language models via self-prompt calibration}.
\newblock \bibinfo{journal}{\emph{arXiv preprint arXiv:2311.06062}} (\bibinfo{year}{2023}).
\newblock


\bibitem[Gao et~al\mbox{.}(2025)]%
        {gao2025federated}
\bibfield{author}{\bibinfo{person}{Zhidong Gao}, \bibinfo{person}{Zhenxiao Zhang}, \bibinfo{person}{Yuanxiong Guo}, {and} \bibinfo{person}{Yanmin Gong}.} \bibinfo{year}{2025}\natexlab{}.
\newblock \showarticletitle{Federated Adaptive Fine-Tuning of Large Language Models with Heterogeneous Quantization and LoRA}. In \bibinfo{booktitle}{\emph{IEEE INFOCOM 2025-IEEE Conference on Computer Communications}}. IEEE, \bibinfo{pages}{1--10}.
\newblock


\bibitem[Houlsby et~al\mbox{.}(2019)]%
        {houlsby2019parameter}
\bibfield{author}{\bibinfo{person}{Neil Houlsby}, \bibinfo{person}{Andrei Giurgiu}, \bibinfo{person}{Stanislaw Jastrzebski}, \bibinfo{person}{Bruna Morrone}, \bibinfo{person}{Quentin De~Laroussilhe}, \bibinfo{person}{Andrea Gesmundo}, \bibinfo{person}{Mona Attariyan}, {and} \bibinfo{person}{Sylvain Gelly}.} \bibinfo{year}{2019}\natexlab{}.
\newblock \showarticletitle{Parameter-efficient transfer learning for NLP}. In \bibinfo{booktitle}{\emph{International conference on machine learning}}. PMLR, \bibinfo{pages}{2790--2799}.
\newblock


\bibitem[Hu et~al\mbox{.}(2022a)]%
        {hu2022lora}
\bibfield{author}{\bibinfo{person}{Edward~J Hu}, \bibinfo{person}{Yelong Shen}, \bibinfo{person}{Phillip Wallis}, \bibinfo{person}{Zeyuan Allen-Zhu}, \bibinfo{person}{Yuanzhi Li}, \bibinfo{person}{Shean Wang}, \bibinfo{person}{Lu Wang}, \bibinfo{person}{Weizhu Chen}, {et~al\mbox{.}}} \bibinfo{year}{2022}\natexlab{a}.
\newblock \showarticletitle{Lora: Low-rank adaptation of large language models.}
\newblock \bibinfo{journal}{\emph{ICLR}} \bibinfo{volume}{1}, \bibinfo{number}{2} (\bibinfo{year}{2022}), \bibinfo{pages}{3}.
\newblock


\bibitem[Hu et~al\mbox{.}(2022b)]%
        {hu2022incentive}
\bibfield{author}{\bibinfo{person}{Miao Hu}, \bibinfo{person}{Di Wu}, \bibinfo{person}{Yipeng Zhou}, \bibinfo{person}{Xu Chen}, {and} \bibinfo{person}{Min Chen}.} \bibinfo{year}{2022}\natexlab{b}.
\newblock \showarticletitle{Incentive-aware autonomous client participation in federated learning}.
\newblock \bibinfo{journal}{\emph{IEEE Transactions on Parallel and Distributed Systems}} \bibinfo{volume}{33}, \bibinfo{number}{10} (\bibinfo{year}{2022}), \bibinfo{pages}{2612--2627}.
\newblock


\bibitem[Hu et~al\mbox{.}(2023)]%
        {hu2023llm}
\bibfield{author}{\bibinfo{person}{Zhiqiang Hu}, \bibinfo{person}{Lei Wang}, \bibinfo{person}{Yihuai Lan}, \bibinfo{person}{Wanyu Xu}, \bibinfo{person}{Ee-Peng Lim}, \bibinfo{person}{Lidong Bing}, \bibinfo{person}{Xing Xu}, \bibinfo{person}{Soujanya Poria}, {and} \bibinfo{person}{Roy Ka-Wei Lee}.} \bibinfo{year}{2023}\natexlab{}.
\newblock \showarticletitle{Llm-adapters: An adapter family for parameter-efficient fine-tuning of large language models}.
\newblock \bibinfo{journal}{\emph{arXiv preprint arXiv:2304.01933}} (\bibinfo{year}{2023}).
\newblock


\bibitem[Koo et~al\mbox{.}(2024)]%
        {koo2024towards}
\bibfield{author}{\bibinfo{person}{Jabin Koo}, \bibinfo{person}{Minwoo Jang}, {and} \bibinfo{person}{Jungseul Ok}.} \bibinfo{year}{2024}\natexlab{}.
\newblock \showarticletitle{Towards robust and efficient federated low-rank adaptation with heterogeneous clients}.
\newblock \bibinfo{journal}{\emph{arXiv preprint arXiv:2410.22815}} (\bibinfo{year}{2024}).
\newblock


\bibitem[Kuo et~al\mbox{.}(2024)]%
        {kuo2024federated}
\bibfield{author}{\bibinfo{person}{Kevin Kuo}, \bibinfo{person}{Arian Raje}, \bibinfo{person}{Kousik Rajesh}, {and} \bibinfo{person}{Virginia Smith}.} \bibinfo{year}{2024}\natexlab{}.
\newblock \showarticletitle{Federated lora with sparse communication}.
\newblock \bibinfo{journal}{\emph{arXiv preprint arXiv:2406.05233}} (\bibinfo{year}{2024}).
\newblock


\bibitem[Lang(1995)]%
        {lang1995newsweeder}
\bibfield{author}{\bibinfo{person}{Ken Lang}.} \bibinfo{year}{1995}\natexlab{}.
\newblock \showarticletitle{Newsweeder: Learning to filter netnews}.
\newblock In \bibinfo{booktitle}{\emph{Machine learning proceedings 1995}}. \bibinfo{publisher}{Elsevier}, \bibinfo{pages}{331--339}.
\newblock


\bibitem[Li et~al\mbox{.}(2024)]%
        {li2024federated}
\bibfield{author}{\bibinfo{person}{Jingyang Li}, \bibinfo{person}{T~Tony Cai}, \bibinfo{person}{Dong Xia}, {and} \bibinfo{person}{Anru~R Zhang}.} \bibinfo{year}{2024}\natexlab{}.
\newblock \showarticletitle{Federated PCA and Estimation for Spiked Covariance Matrices: Optimal Rates and Efficient Algorithm}.
\newblock \bibinfo{journal}{\emph{arXiv preprint arXiv:2411.15660}} (\bibinfo{year}{2024}).
\newblock


\bibitem[Li and Liang(2021)]%
        {li2021prefix}
\bibfield{author}{\bibinfo{person}{Xiang~Lisa Li} {and} \bibinfo{person}{Percy Liang}.} \bibinfo{year}{2021}\natexlab{}.
\newblock \showarticletitle{Prefix-tuning: Optimizing continuous prompts for generation}.
\newblock \bibinfo{journal}{\emph{arXiv preprint arXiv:2101.00190}} (\bibinfo{year}{2021}).
\newblock


\bibitem[Lin et~al\mbox{.}(2023)]%
        {lin2023heterogeneous}
\bibfield{author}{\bibinfo{person}{Xi Lin}, \bibinfo{person}{Jun Wu}, \bibinfo{person}{Jianhua Li}, \bibinfo{person}{Chao Sang}, \bibinfo{person}{Shiyan Hu}, {and} \bibinfo{person}{M~Jamal Deen}.} \bibinfo{year}{2023}\natexlab{}.
\newblock \showarticletitle{Heterogeneous differential-private federated learning: Trading privacy for utility truthfully}.
\newblock \bibinfo{journal}{\emph{IEEE Transactions on Dependable and Secure Computing}} \bibinfo{volume}{20}, \bibinfo{number}{6} (\bibinfo{year}{2023}), \bibinfo{pages}{5113--5129}.
\newblock


\bibitem[Liu et~al\mbox{.}(2024)]%
        {liu2024decentralized}
\bibfield{author}{\bibinfo{person}{Yingqi Liu}, \bibinfo{person}{Yifan Shi}, \bibinfo{person}{Qinglun Li}, \bibinfo{person}{Baoyuan Wu}, \bibinfo{person}{Xueqian Wang}, {and} \bibinfo{person}{Li Shen}.} \bibinfo{year}{2024}\natexlab{}.
\newblock \showarticletitle{Decentralized directed collaboration for personalized federated learning}. In \bibinfo{booktitle}{\emph{Proceedings of the IEEE/CVF conference on computer vision and pattern recognition}}. \bibinfo{pages}{23168--23178}.
\newblock


\bibitem[Lotfi et~al\mbox{.}(2024)]%
        {lotfi2024unlocking}
\bibfield{author}{\bibinfo{person}{Sanae Lotfi}, \bibinfo{person}{Yilun Kuang}, \bibinfo{person}{Marc Finzi}, \bibinfo{person}{Brandon Amos}, \bibinfo{person}{Micah Goldblum}, {and} \bibinfo{person}{Andrew~G Wilson}.} \bibinfo{year}{2024}\natexlab{}.
\newblock \showarticletitle{Unlocking tokens as data points for generalization bounds on larger language models}.
\newblock \bibinfo{journal}{\emph{Advances in Neural Information Processing Systems}}  \bibinfo{volume}{37} (\bibinfo{year}{2024}), \bibinfo{pages}{9229--9256}.
\newblock


\bibitem[Lukas et~al\mbox{.}(2023)]%
        {lukas2023analyzing}
\bibfield{author}{\bibinfo{person}{Nils Lukas}, \bibinfo{person}{Ahmed Salem}, \bibinfo{person}{Robert Sim}, \bibinfo{person}{Shruti Tople}, \bibinfo{person}{Lukas Wutschitz}, {and} \bibinfo{person}{Santiago Zanella-B{\'e}guelin}.} \bibinfo{year}{2023}\natexlab{}.
\newblock \showarticletitle{Analyzing leakage of personally identifiable information in language models}. In \bibinfo{booktitle}{\emph{2023 IEEE Symposium on Security and Privacy (SP)}}. IEEE, \bibinfo{pages}{346--363}.
\newblock


\bibitem[Luo et~al\mbox{.}(2023)]%
        {luo2023incentive}
\bibfield{author}{\bibinfo{person}{Bing Luo}, \bibinfo{person}{Yutong Feng}, \bibinfo{person}{Shiqiang Wang}, \bibinfo{person}{Jianwei Huang}, {and} \bibinfo{person}{Leandros Tassiulas}.} \bibinfo{year}{2023}\natexlab{}.
\newblock \showarticletitle{Incentive mechanism design for unbiased federated learning with randomized client participation}. In \bibinfo{booktitle}{\emph{2023 IEEE 43rd International Conference on Distributed Computing Systems (ICDCS)}}. IEEE, \bibinfo{pages}{545--555}.
\newblock


\bibitem[Mao et~al\mbox{.}(2024)]%
        {mao2024game}
\bibfield{author}{\bibinfo{person}{Wuxing Mao}, \bibinfo{person}{Qian Ma}, \bibinfo{person}{Guocheng Liao}, {and} \bibinfo{person}{Xu Chen}.} \bibinfo{year}{2024}\natexlab{}.
\newblock \showarticletitle{Game analysis and incentive mechanism design for differentially private cross-silo federated learning}.
\newblock \bibinfo{journal}{\emph{IEEE Transactions on Mobile Computing}} \bibinfo{volume}{23}, \bibinfo{number}{10} (\bibinfo{year}{2024}), \bibinfo{pages}{9337--9351}.
\newblock


\bibitem[McMahan et~al\mbox{.}(2017)]%
        {mcmahan2017communication}
\bibfield{author}{\bibinfo{person}{Brendan McMahan}, \bibinfo{person}{Eider Moore}, \bibinfo{person}{Daniel Ramage}, \bibinfo{person}{Seth Hampson}, {and} \bibinfo{person}{Blaise~Aguera y Arcas}.} \bibinfo{year}{2017}\natexlab{}.
\newblock \showarticletitle{Communication-efficient learning of deep networks from decentralized data}. In \bibinfo{booktitle}{\emph{Artificial intelligence and statistics}}. PMLR, \bibinfo{pages}{1273--1282}.
\newblock


\bibitem[Nguyen et~al\mbox{.}(2024)]%
        {nguyen2024towards}
\bibfield{author}{\bibinfo{person}{Ngoc-Hieu Nguyen}, \bibinfo{person}{Tuan-Anh Nguyen}, \bibinfo{person}{Tuan Nguyen}, \bibinfo{person}{Vu~Tien Hoang}, \bibinfo{person}{Dung~D Le}, {and} \bibinfo{person}{Kok-Seng Wong}.} \bibinfo{year}{2024}\natexlab{}.
\newblock \showarticletitle{Towards efficient communication and secure federated recommendation system via low-rank training}. In \bibinfo{booktitle}{\emph{Proceedings of the ACM Web Conference 2024}}. \bibinfo{pages}{3940--3951}.
\newblock


\bibitem[Radford et~al\mbox{.}(2019)]%
        {radford2019language}
\bibfield{author}{\bibinfo{person}{Alec Radford}, \bibinfo{person}{Jeffrey Wu}, \bibinfo{person}{Rewon Child}, \bibinfo{person}{David Luan}, \bibinfo{person}{Dario Amodei}, \bibinfo{person}{Ilya Sutskever}, {et~al\mbox{.}}} \bibinfo{year}{2019}\natexlab{}.
\newblock \showarticletitle{Language models are unsupervised multitask learners}.
\newblock \bibinfo{journal}{\emph{OpenAI blog}} \bibinfo{volume}{1}, \bibinfo{number}{8} (\bibinfo{year}{2019}), \bibinfo{pages}{9}.
\newblock


\bibitem[Rajbhandari et~al\mbox{.}(2020)]%
        {rajbhandari2020zero}
\bibfield{author}{\bibinfo{person}{Samyam Rajbhandari}, \bibinfo{person}{Jeff Rasley}, \bibinfo{person}{Olatunji Ruwase}, {and} \bibinfo{person}{Yuxiong He}.} \bibinfo{year}{2020}\natexlab{}.
\newblock \showarticletitle{ZeRO: Memory Optimizations Toward Training Trillion Parameter Models}. In \bibinfo{booktitle}{\emph{SC '20: International Conference for High Performance Computing, Networking, Storage and Analysis}}. \bibinfo{publisher}{IEEE}.
\newblock
\urldef\tempurl%
\url{https://sc20.supercomputing.org/proceedings/tech_paper/tech_paper_pages/pap379.html}
\showURL{%
\tempurl}


\bibitem[Rashid et~al\mbox{.}(2025)]%
        {rashid2025trustworthy}
\bibfield{author}{\bibinfo{person}{Md~Mamunur Rashid}, \bibinfo{person}{Yong Xiang}, \bibinfo{person}{Md~Palash Uddin}, \bibinfo{person}{Jine Tang}, \bibinfo{person}{Keshav Sood}, {and} \bibinfo{person}{Longxiang Gao}.} \bibinfo{year}{2025}\natexlab{}.
\newblock \showarticletitle{Trustworthy and fair federated learning via reputation-based consensus and adaptive incentives}.
\newblock \bibinfo{journal}{\emph{IEEE Transactions on Information Forensics and Security}} (\bibinfo{year}{2025}).
\newblock


\bibitem[Sun et~al\mbox{.}(2024)]%
        {sun2024improving}
\bibfield{author}{\bibinfo{person}{Youbang Sun}, \bibinfo{person}{Zitao Li}, \bibinfo{person}{Yaliang Li}, {and} \bibinfo{person}{Bolin Ding}.} \bibinfo{year}{2024}\natexlab{}.
\newblock \showarticletitle{Improving lora in privacy-preserving federated learning}.
\newblock \bibinfo{journal}{\emph{arXiv preprint arXiv:2403.12313}} (\bibinfo{year}{2024}).
\newblock


\bibitem[Tang et~al\mbox{.}(2024)]%
        {tang2024merit}
\bibfield{author}{\bibinfo{person}{Yongyang Tang}, \bibinfo{person}{Zhe Chen}, \bibinfo{person}{Ang Li}, \bibinfo{person}{Tianyue Zheng}, \bibinfo{person}{Zheng Lin}, \bibinfo{person}{Jia Xu}, \bibinfo{person}{Pin Lv}, \bibinfo{person}{Zhe Sun}, {and} \bibinfo{person}{Yue Gao}.} \bibinfo{year}{2024}\natexlab{}.
\newblock \showarticletitle{Merit: Multimodal wearable vital sign waveform monitoring}. In \bibinfo{booktitle}{\emph{2024 IEEE Smart World Congress (SWC)}}. IEEE, \bibinfo{pages}{1112--1119}.
\newblock


\bibitem[Team et~al\mbox{.}(2023)]%
        {team2023gemini}
\bibfield{author}{\bibinfo{person}{Gemini Team}, \bibinfo{person}{Rohan Anil}, \bibinfo{person}{Sebastian Borgeaud}, \bibinfo{person}{Jean-Baptiste Alayrac}, \bibinfo{person}{Jiahui Yu}, \bibinfo{person}{Radu Soricut}, \bibinfo{person}{Johan Schalkwyk}, \bibinfo{person}{Andrew~M Dai}, \bibinfo{person}{Anja Hauth}, \bibinfo{person}{Katie Millican}, {et~al\mbox{.}}} \bibinfo{year}{2023}\natexlab{}.
\newblock \showarticletitle{Gemini: a family of highly capable multimodal models}.
\newblock \bibinfo{journal}{\emph{arXiv preprint arXiv:2312.11805}} (\bibinfo{year}{2023}).
\newblock


\bibitem[Touvron et~al\mbox{.}(2023)]%
        {touvron2023llama}
\bibfield{author}{\bibinfo{person}{Hugo Touvron}, \bibinfo{person}{Thibaut Lavril}, \bibinfo{person}{Gautier Izacard}, \bibinfo{person}{Xavier Martinet}, \bibinfo{person}{Marie-Anne Lachaux}, \bibinfo{person}{Timoth{\'e}e Lacroix}, \bibinfo{person}{Baptiste Rozi{\`e}re}, \bibinfo{person}{Naman Goyal}, \bibinfo{person}{Eric Hambro}, \bibinfo{person}{Faisal Azhar}, {et~al\mbox{.}}} \bibinfo{year}{2023}\natexlab{}.
\newblock \showarticletitle{Llama: Open and efficient foundation language models}.
\newblock \bibinfo{journal}{\emph{arXiv preprint arXiv:2302.13971}} (\bibinfo{year}{2023}).
\newblock


\bibitem[Wang et~al\mbox{.}(2022)]%
        {wang2022privaim}
\bibfield{author}{\bibinfo{person}{Dan Wang}, \bibinfo{person}{Ju Ren}, \bibinfo{person}{Zhibo Wang}, \bibinfo{person}{Yichuan Wang}, {and} \bibinfo{person}{Yaoxue Zhang}.} \bibinfo{year}{2022}\natexlab{}.
\newblock \showarticletitle{Privaim: A dual-privacy preserving and quality-aware incentive mechanism for federated learning}.
\newblock \bibinfo{journal}{\emph{IEEE Trans. Comput.}} \bibinfo{volume}{72}, \bibinfo{number}{7} (\bibinfo{year}{2022}), \bibinfo{pages}{1913--1927}.
\newblock


\bibitem[Wang et~al\mbox{.}(2024c)]%
        {wang2024balancing}
\bibfield{author}{\bibinfo{person}{Naiyu Wang}, \bibinfo{person}{Shen Wang}, \bibinfo{person}{Meng Li}, \bibinfo{person}{Longfei Wu}, \bibinfo{person}{Zijian Zhang}, \bibinfo{person}{Zhitao Guan}, {and} \bibinfo{person}{Liehuang Zhu}.} \bibinfo{year}{2024}\natexlab{c}.
\newblock \showarticletitle{Balancing differential privacy and utility: A relevance-based adaptive private fine-tuning framework for language models}.
\newblock \bibinfo{journal}{\emph{IEEE Transactions on Information Forensics and Security}} (\bibinfo{year}{2024}).
\newblock


\bibitem[Wang et~al\mbox{.}(2024a)]%
        {wang2024flora}
\bibfield{author}{\bibinfo{person}{Ziyao Wang}, \bibinfo{person}{Zheyu Shen}, \bibinfo{person}{Yexiao He}, \bibinfo{person}{Guoheng Sun}, \bibinfo{person}{Hongyi Wang}, \bibinfo{person}{Lingjuan Lyu}, {and} \bibinfo{person}{Ang Li}.} \bibinfo{year}{2024}\natexlab{a}.
\newblock \showarticletitle{Flora: Federated fine-tuning large language models with heterogeneous low-rank adaptations}.
\newblock \bibinfo{journal}{\emph{Advances in Neural Information Processing Systems}}  \bibinfo{volume}{37} (\bibinfo{year}{2024}), \bibinfo{pages}{22513--22533}.
\newblock


\bibitem[Wang et~al\mbox{.}(2024b)]%
        {wang2024one}
\bibfield{author}{\bibinfo{person}{Ziyao Wang}, \bibinfo{person}{Bowei Tian}, \bibinfo{person}{Yexiao He}, \bibinfo{person}{Zheyu Shen}, \bibinfo{person}{Luyang Liu}, {and} \bibinfo{person}{Ang Li}.} \bibinfo{year}{2024}\natexlab{b}.
\newblock \showarticletitle{One Communication Round is All It Needs for Federated Fine-Tuning Foundation Models}.
\newblock \bibinfo{journal}{\emph{arXiv preprint arXiv:2412.04650}} (\bibinfo{year}{2024}).
\newblock


\bibitem[Wang et~al\mbox{.}(2025b)]%
        {wang2025KGEES}
\bibfield{author}{\bibinfo{person}{Ziqi Wang}, \bibinfo{person}{Xiaoyu Xia}, \bibinfo{person}{Ibrahim Khalil}, \bibinfo{person}{Minghui Liwang}, \bibinfo{person}{Xiaolong Xu}, \bibinfo{person}{Xun Yi}, \bibinfo{person}{Yan Li}, {and} \bibinfo{person}{Minhui Xue}.} \bibinfo{year}{2025}\natexlab{b}.
\newblock \showarticletitle{{ KGEES: An Energy Saving System with Location Privacy Preservation in Multi-Access Edge Computing }}.
\newblock \bibinfo{journal}{\emph{IEEE Transactions on Dependable and Secure Computing}} \bibinfo{number}{01} (\bibinfo{date}{Dec.} \bibinfo{year}{2025}), \bibinfo{pages}{1--14}.
\newblock
\showISSN{1941-0018}


\bibitem[Wang et~al\mbox{.}(2025a)]%
        {wang2025MoSEEC}
\bibfield{author}{\bibinfo{person}{Ziqi Wang}, \bibinfo{person}{Xiaoyu Xia}, \bibinfo{person}{Ibrahim Khalil}, \bibinfo{person}{Minghui Liwang}, {and} \bibinfo{person}{Minhui Xue}.} \bibinfo{year}{2025}\natexlab{a}.
\newblock \showarticletitle{{ MoSEEC: Sustainable and Trajectory Privacy-Preserving Edge Resource Management }}.
\newblock \bibinfo{journal}{\emph{IEEE Transactions on Mobile Computing}} \bibinfo{number}{01} (\bibinfo{date}{Nov.} \bibinfo{year}{2025}), \bibinfo{pages}{1--14}.
\newblock
\showISSN{1558-0660}


\bibitem[Wang et~al\mbox{.}(2024d)]%
        {wang2024gees}
\bibfield{author}{\bibinfo{person}{Ziqi Wang}, \bibinfo{person}{Xiaoyu Xia}, \bibinfo{person}{Minhui Xue}, \bibinfo{person}{Ibrahim Khalil}, \bibinfo{person}{Minghui Liwang}, {and} \bibinfo{person}{Xun Yi}.} \bibinfo{year}{2024}\natexlab{d}.
\newblock \showarticletitle{GEES: Enabling Location Privacy-Preserving Energy Saving in Multi-Access Edge Computing}. In \bibinfo{booktitle}{\emph{Proceedings of the ACM on Web Conference 2024}}. \bibinfo{pages}{2735--2746}.
\newblock


\bibitem[Wei et~al\mbox{.}(2020)]%
        {wei2020federated}
\bibfield{author}{\bibinfo{person}{Kang Wei}, \bibinfo{person}{Jun Li}, \bibinfo{person}{Ming Ding}, \bibinfo{person}{Chuan Ma}, \bibinfo{person}{Howard~H Yang}, \bibinfo{person}{Farhad Farokhi}, \bibinfo{person}{Shi Jin}, \bibinfo{person}{Tony~QS Quek}, {and} \bibinfo{person}{H~Vincent Poor}.} \bibinfo{year}{2020}\natexlab{}.
\newblock \showarticletitle{Federated learning with differential privacy: Algorithms and performance analysis}.
\newblock \bibinfo{journal}{\emph{IEEE transactions on information forensics and security}}  \bibinfo{volume}{15} (\bibinfo{year}{2020}), \bibinfo{pages}{3454--3469}.
\newblock


\bibitem[Wen et~al\mbox{.}(2024)]%
        {wen2024membership}
\bibfield{author}{\bibinfo{person}{Rui Wen}, \bibinfo{person}{Zheng Li}, \bibinfo{person}{Michael Backes}, {and} \bibinfo{person}{Yang Zhang}.} \bibinfo{year}{2024}\natexlab{}.
\newblock \showarticletitle{Membership inference attacks against in-context learning}. In \bibinfo{booktitle}{\emph{Proceedings of the 2024 on ACM SIGSAC Conference on Computer and Communications Security}}. \bibinfo{pages}{3481--3495}.
\newblock


\bibitem[Wu et~al\mbox{.}(2023)]%
        {wu2023bloomberggpt}
\bibfield{author}{\bibinfo{person}{Shijie Wu}, \bibinfo{person}{Ozan Irsoy}, \bibinfo{person}{Steven Lu}, \bibinfo{person}{Vadim Dabravolski}, \bibinfo{person}{Mark Dredze}, \bibinfo{person}{Sebastian Gehrmann}, \bibinfo{person}{Prabhanjan Kambadur}, \bibinfo{person}{David Rosenberg}, {and} \bibinfo{person}{Gideon Mann}.} \bibinfo{year}{2023}\natexlab{}.
\newblock \showarticletitle{Bloomberggpt: A large language model for finance}.
\newblock \bibinfo{journal}{\emph{arXiv preprint arXiv:2303.17564}} (\bibinfo{year}{2023}).
\newblock


\bibitem[Xia et~al\mbox{.}(2025)]%
        {xia2025edge}
\bibfield{author}{\bibinfo{person}{Xiaoyu Xia}, \bibinfo{person}{Ziqi Wang}, \bibinfo{person}{Ruoxi Sun}, \bibinfo{person}{Bowen Liu}, \bibinfo{person}{Ibrahim Khalil}, {and} \bibinfo{person}{Minhui Xue}.} \bibinfo{year}{2025}\natexlab{}.
\newblock \showarticletitle{Edge Unlearning is Not “on Edge”! An Adaptive Exact Unlearning System on Resource-Constrained Devices}. In \bibinfo{booktitle}{\emph{2025 IEEE Symposium on Security and Privacy (SP)}}. IEEE, \bibinfo{pages}{2546--2563}.
\newblock


\bibitem[Xu et~al\mbox{.}(2025)]%
        {xu2025dp}
\bibfield{author}{\bibinfo{person}{Honghui Xu}, \bibinfo{person}{Shiva Shrestha}, \bibinfo{person}{Wei Chen}, \bibinfo{person}{Zhiyuan Li}, {and} \bibinfo{person}{Zhipeng Cai}.} \bibinfo{year}{2025}\natexlab{}.
\newblock \showarticletitle{DP-FedLoRA: Privacy-Enhanced Federated Fine-Tuning for On-Device Large Language Models}.
\newblock \bibinfo{journal}{\emph{arXiv preprint arXiv:2509.09097}} (\bibinfo{year}{2025}).
\newblock


\bibitem[Yan et~al\mbox{.}(2024)]%
        {yan2024federa}
\bibfield{author}{\bibinfo{person}{Yuxuan Yan}, \bibinfo{person}{Qianqian Yang}, \bibinfo{person}{Shunpu Tang}, {and} \bibinfo{person}{Zhiguo Shi}.} \bibinfo{year}{2024}\natexlab{}.
\newblock \showarticletitle{Federa: Efficient fine-tuning of language models in federated learning leveraging weight decomposition}.
\newblock \bibinfo{journal}{\emph{arXiv preprint arXiv:2404.18848}} (\bibinfo{year}{2024}).
\newblock


\bibitem[Yang et~al\mbox{.}(2023)]%
        {yang2023csra}
\bibfield{author}{\bibinfo{person}{Yunchao Yang}, \bibinfo{person}{Miao Hu}, \bibinfo{person}{Yipeng Zhou}, \bibinfo{person}{Xuezheng Liu}, {and} \bibinfo{person}{Di Wu}.} \bibinfo{year}{2023}\natexlab{}.
\newblock \showarticletitle{Csra: Robust incentive mechanism design for differentially private federated learning}.
\newblock \bibinfo{journal}{\emph{IEEE Transactions on Information Forensics and Security}}  \bibinfo{volume}{19} (\bibinfo{year}{2023}), \bibinfo{pages}{892--906}.
\newblock


\bibitem[Zhang and Zhou(2024)]%
        {zhang2024leave}
\bibfield{author}{\bibinfo{person}{Anderson~Y Zhang} {and} \bibinfo{person}{Harrison~Y Zhou}.} \bibinfo{year}{2024}\natexlab{}.
\newblock \showarticletitle{Leave-one-out singular subspace perturbation analysis for spectral clustering}.
\newblock \bibinfo{journal}{\emph{The Annals of Statistics}} \bibinfo{volume}{52}, \bibinfo{number}{5} (\bibinfo{year}{2024}), \bibinfo{pages}{2004--2033}.
\newblock


\bibitem[Zhang et~al\mbox{.}(2024a)]%
        {zhang2024towards}
\bibfield{author}{\bibinfo{person}{Jianyi Zhang}, \bibinfo{person}{Saeed Vahidian}, \bibinfo{person}{Martin Kuo}, \bibinfo{person}{Chunyuan Li}, \bibinfo{person}{Ruiyi Zhang}, \bibinfo{person}{Tong Yu}, \bibinfo{person}{Guoyin Wang}, {and} \bibinfo{person}{Yiran Chen}.} \bibinfo{year}{2024}\natexlab{a}.
\newblock \showarticletitle{Towards building the federatedgpt: Federated instruction tuning}. In \bibinfo{booktitle}{\emph{ICASSP 2024-2024 IEEE International Conference on Acoustics, Speech and Signal Processing (ICASSP)}}. IEEE, \bibinfo{pages}{6915--6919}.
\newblock


\bibitem[Zhang et~al\mbox{.}(2021)]%
        {zhang2021incentive}
\bibfield{author}{\bibinfo{person}{Jingwen Zhang}, \bibinfo{person}{Yuezhou Wu}, {and} \bibinfo{person}{Rong Pan}.} \bibinfo{year}{2021}\natexlab{}.
\newblock \showarticletitle{Incentive mechanism for horizontal federated learning based on reputation and reverse auction}. In \bibinfo{booktitle}{\emph{Proceedings of the Web Conference 2021}}. \bibinfo{pages}{947--956}.
\newblock


\bibitem[Zhang et~al\mbox{.}(2024b)]%
        {zhang2024tinyllama}
\bibfield{author}{\bibinfo{person}{Peiyuan Zhang}, \bibinfo{person}{Guangtao Zeng}, \bibinfo{person}{Tianduo Wang}, {and} \bibinfo{person}{Wei Lu}.} \bibinfo{year}{2024}\natexlab{b}.
\newblock \showarticletitle{Tinyllama: An open-source small language model}.
\newblock \bibinfo{journal}{\emph{arXiv preprint arXiv:2401.02385}} (\bibinfo{year}{2024}).
\newblock


\bibitem[Zhang et~al\mbox{.}(2015)]%
        {zhang2015character}
\bibfield{author}{\bibinfo{person}{Xiang Zhang}, \bibinfo{person}{Junbo Zhao}, {and} \bibinfo{person}{Yann LeCun}.} \bibinfo{year}{2015}\natexlab{}.
\newblock \showarticletitle{Character-level convolutional networks for text classification}.
\newblock \bibinfo{journal}{\emph{Advances in neural information processing systems}}  \bibinfo{volume}{28} (\bibinfo{year}{2015}).
\newblock


\end{thebibliography}

\appendix
\nobalance
\section*{Appendix} 
\crefalias{section}{appendix}
\crefalias{subsection}{appendix}
\section{Summary of Notations}
\label{appendix:notations}
To facilitate clarity and improve readability, we summarize the key notations used throughout the paper in \Cref{tab:summary_of_notations}. 
\begin{table}[]
\renewcommand{\arraystretch}{1.2}
\footnotesize
\caption{Summary of Notations}
\label{tab:summary_of_notations}
\centering
\begin{tabular}{l|l}
\hline
Notation & Description\\
\hline
$\mathrm{A}_i^t$ & Client $c_i$'s LoRA $A$ matrices at $t$\\
$\mathcal{A}(\cdot)$ & A performance metric for global model\\
$\alpha_{dir}$ & Dirichlet concentration parameter \\
$\mathrm{B}_i^t$ & Client $c_i$'s LoRA $B$ matrices at $t$\\
$\beta$ & Decay of exponential moving average \\
$c_i \in C$ & Client $c_i$\\
$\mathcal{D}^t$ & Global datasets across all clients at $t$\\
$\mathcal{D}_i^t$ & Local dataset of $c_i$ at $t$ \\
$d_B$ & Dimension of the concatenated and vectorized LoRA $B$ \\
$\Delta W_g^t$ & Aggregated global parameter updates \\
$\Delta W_i^t$ & Parameter updates of client $c_i$ at $t$\\
$\gamma_i$ & Privacy preference of $c_i$\\
$I_k^t$ & Upper confidence bound for noise option $\sigma^{(k)}$ at $t$\\
$\mathcal{M}$& Incentive mechanism\\
$\mu^{(-i)}$ & Row-mean of $X^{(-i)}$ \\
$\hat\mu_k$ & Estimate of expected utility\\
$N$ & Number of clients \\
$P_i$ & Projector operator \\
$\mathcal{P}_i$ & Privacy term in $U_i$ \\
$\pi$ & Strategy profile of all clients $\{\pi_i\}_{i=1}^N$ in the SAMG \\
$\pi^*$ & Stationary Markov equilibrium for the SAMG\\
$\rho$ & Bias factor for estimate bias\\
$\|r_i\|_2^2$ & Residual energy norm of $c_i$ \\
$\Sigma$ & Noise action set, $\{\sigma^{(1)}, \cdots,\sigma^{(\mathbb{K})}\}$ \\
$s_{i}^t$ & Client $c_i$'s score evaluated by the server at $t$ \\
$\boldsymbol{\sigma}^t$& Joint noise profile of all clients at $t$\\
$\hat{\sigma}_i^t$ & Estimated noise scale of $c_i$ at $t$\\
$\tilde{\sigma}_i^t$ & Estimated noise scale of $c_i$ at $t$ with bias\\
${\sigma}^t_i$ & Noise scale of $c_i$ at $t$\\
$T$ & Total communication rounds \\
$t$ & Communication round \\
$U_i^t$ & Client $c_i$'s utility at $t$\\
$U_s^t$ & System utility at $t$ \\
$\overline{U}_i$ & Client $c_i$'s long-term utility in the SAMG\\
$w_i^t$ & Aggregation weight assigned to client $c_i$ at $t$\\
$\mathrm{W}$ & Pre-trained base model \\
$\mathrm{W}_g^t$ & Global model at $t$ \\
$\mathcal W$ & State space in the SAMG\\
$X^{(-i)}$ & Flattened and concatenated LoRA matrices of all clients except $c_i$ \\
\hline
\end{tabular}
\end{table}

\section{Details of Algorithm}
This section demonstrate the details of~\cref{alg:loo_pca_b} and~\cref{alg:noise_aware_allocation} mentioned in~\cref{subsec:NWA} and ~\cref{alg:client_best_response} mentioned in~\cref{subsec:INA}.
\label{appendix:alg}
\begin{algorithm}[]
\caption{LOO–PCA Residual Noise Estimation on LoRA-$B$}
\label{alg:loo_pca_b}
\SetAlgoLined
\KwIn{Client LoRA updates $\{\Delta A_i\,\Delta B_i\}_{i=1}^N$; Number of each $\Delta B_i$'s modules $\mathcal{K}$}
\KwOut{Per–client noise estimates $\hat{\sigma}_1,\dots,\hat{\sigma}_N$}

\Comment{Construct matrix} 
For each client $i=1,\ldots,N$ we collect all LoRA-B matrices and stack them after vectorization:\\ \nllabel{line:stack}
$x_i \;=\; \bigoplus_{\,b\in\mathcal{K}}\mathrm{vec}\!\big(B_i^{(b)}\big)\;\in\;\mathbb{R}^{d_B},
\qquad
d_B \;=\; \sum_{b\in\mathcal{K}} r\, d_{out}$.
Form $X \leftarrow [\,x_1,\dots,x_N\,] \in \mathbb{R}^{d_B\times N}$.

\For{$i=1$ \KwTo $N$}{
  \Comment{Form public subspace and projector operator}
  $X^{(-i)} \leftarrow X$ with column $i$ removed $\in \mathbb{R}^{d_{out}\times(N-1)}$;\nllabel{line:formX} \\ 
  $\mu^{(-i)} \leftarrow \text{row-mean}(X^{(-i)}) \in \mathbb{R}^{d_B}$; \nllabel{line:mu}\\
  \textbf{Center} $X^{(-i)} \leftarrow X^{(-i)} - \mu^{(-i)} \cdot \mathbf{1}^\top$; \nllabel{line:center}\\
  \textbf{Compute} randomized SVD: $X^{(-i)} = U_{-i}\,\Sigma_{-i}\,V_{-i}^\top$; \nllabel{line:SVD} \\
  $\mathrm{rank}~K \leftarrow \mathrm{cols}(U)-1$; \\
  $U_K \leftarrow U[:,\,1{:}K] \in \mathbb{R}^{d_B\times K}$; \\
  Projector operator $P_i \;=\; U^{(-i)}_{K}\big(U^{(-i)}_{K}\big)^\top$ \nllabel{line:Proj} \\
  \Comment{Residual energy as noise estimates}
  $r_i \leftarrow \big(I -P_i\big)\,(x_i - \mu^{(-i)})$; \nllabel{line:residual} \\
  $\hat{\sigma}_i^2 \leftarrow \|r_i\|_2^2 / (d_B - K)$;\quad $\hat{\sigma}_i \leftarrow \sqrt{\hat{\sigma}_i^2}$; \nllabel{line:estimates}
}
\Return $\{\hat{\sigma}_i\}_{i=1}^N$.
\end{algorithm}

\begin{algorithm}[]
\caption{Noise-aware Weights Allocation (NWA)}
\label{alg:noise_aware_allocation}

\KwIn{Global $W_g^t$; Clients $\{c_1,\dots,c_N\}$; client updates $\{(A_i^t,B_i^t)\}$.}
\KwOut{Aggregation weights $\{w_i^t\}_{i=1}^{N}$; global update $\Delta W_g^t$; global model $W_g^{t+1}$.}
\For{$i=1$ \KwTo $C$}{
    Estimate noise: $\hat\sigma_{i}^t \leftarrow \text{LOO-PCA}(B_i^t)$ in~\Cref{alg:loo_pca_b}\nllabel{line:loo}\\
    Map to score: $s_{i}^t \leftarrow \big(1/(\hat\sigma_{i}^t+\tau)\big)$\; \nllabel{line:score}
}
\textbf{Allocation weights to $c_i$:} $w_{i}^t \leftarrow s_{i}^t/\sum_{i=1}^{N} s_{i}^t$ \nllabel{line:weights}

\textbf{Aggregation for global $\Delta W_g^t$:}
$\Delta W_g^t \leftarrow \sum_{i=1}^{N} w_{i}^t\, B_i^tA_i^t$ \nllabel{line:aggre}

\textbf{return}  $W_g^{t+1} \leftarrow W_g^t + \Delta W_g^t$
\end{algorithm}

\begin{algorithm}[]
\caption{Individual Noise Adaptation (INA) based on UCB}
\label{alg:client_best_response}
\KwIn{Noise action set $\Sigma=\{\sigma^{(1)},\ldots,\sigma^{(\mathbb{K})}\}$; communication rounds $T$; exploration scale $\kappa>0$; EMA decay $\rho\in[0,1]$; privacy preference $\gamma_i$}
\KwOut{Noise sequence of $c_i$ over rounds $\{\sigma_i^t\}_{t=1}^T$}

\textbf{Initialize} For $k=1$ to $\mathbb K$: set $n_k\leftarrow 0$ and $\hat\mu_k\leftarrow \mu_0$

\For{$t=1$ \KwTo $T$}{
  \For{$k\gets 1$ \KwTo $\mathbb K$}{
    Compute UCB indices as~\cref{eq:UCB}\; \nllabel{line:compute}
  }
  \textbf{Select} $k \in \arg\max_{1\le k\le \mathbb{K}} I_k^t$ and set $\sigma_i^t \leftarrow \sigma^{(k)}$ \nllabel{line:select}\\
  \textbf{Upload} the locally perturbed update using $\sigma_i^t$ \nllabel{line:upload}\\
  \textbf{Observe} the new global model $W_g^{t}$ and evaluate local performance $\mathcal{G}_i^t\in[0,1]$ \nllabel{line:observe}\\
  \textbf{Compute utility} $U_i^t$ as~\cref{eq:client_utility} \nllabel{line:utility}\\
  \textbf{Update} $n_{k}\leftarrow n_{k}+1$ and $\hat\mu_{k}\leftarrow (1-\beta)\hat\mu_{k} + \beta\, U_i^t$ \nllabel{line:update}
}
\end{algorithm}

\section{Theoretical Analysis}
\label{appendix: theoretical_analysis}
We consider a \textbf{stochastic aggregative Markov game (SAMG)} with client set $\mathcal C=\{c_1,\ldots,c_N\}$ and state space $\mathcal W$. The state (the global model) $W_g\in\mathcal W$ evolves in rounds. Each client $c_i$ selects a noise level from a finite noise action set $\Sigma$ according to a stationary Markov mixed strategy $\pi_i:\mathcal W\to\mathsf{Prob}(\Sigma)$. Let $\pi=(\pi_1,\ldots,\pi_N)$ denote the strategy profile. Conditional on $W_g$, clients act independently; hence the joint action distribution on $\Sigma^N$ is $\bigotimes_{c_j\in\mathcal C}\pi_j(W_g)$, and a realized joint action is a sample $\sigma=(\sigma_1,\ldots,\sigma_N)\in\Sigma^N$ drawn from this distribution.

Each client’s local update is summarized by a bounded, continuous map $Q_i:\mathcal W\times\Sigma\to\mathcal H$, where $\mathcal H$ is a Hilbert space. The server forms an aggregate statistic $\Gamma(W_g,\sigma)=\sum_i w_i(\sigma)\,Q_i(W_g,\sigma_i)$ with inverse-noise weights
$w_i(\sigma)=\frac{1/\sigma_i+\tau}{\sum_j(1/\sigma_j+\tau)}$.
Utilities and the state kernel $\mathcal F$ depend on actions only through $\Gamma$, i.e. ,
$U_i(W_g,\sigma)=\mathcal{G}_i\!\big(W_g,\Gamma(W_g,\sigma)\big)+\gamma_i P_i(\sigma_i)$ and
$\mathcal F\big(\cdot\,\big|\,W_g,\Gamma(W_g,\sigma)\big)$.

We then define the long-term average utility of client $c_i$ as:
\begin{equation}
\overline{U}_i(\pi)\;=\;\lim_{T\to\infty}\frac{1}{T}\,\mathbb E_\pi\!\Big[\sum_{t=1}^T U_i^t\big(W_g^t,\sigma^t\big)\Big].
\end{equation}

\noindent\textbf{Equilibrium and existence.} 
An equilibrium of the SAMG is reached when no client can unilaterally switch to any alternative stationary Markov strategy $\pi_i'$ and obtain a higher long-term average utility, holding others fixed at $\pi_{-i}^*$. Formally, a stationary Markov equilibrium (SME) is defined as:
\begin{definition}[Stationary Markov equilibrium]
A stationary profile $\pi^*=(\pi_1^*,\ldots,\pi_N^*)$ is an equilibrium if, for every client $c_i$
and every admissible alternative stationary Markov strategy $\pi_i':\mathcal W\to\mathsf{Prob}(\Sigma)$,
\begin{equation}\label{eq:equilibrium}
\overline{U}_i(\pi^*)\ \ge\ \overline{U}_i\big((\pi_i',\pi_{-i}^*)\big).
\end{equation}
\end{definition}
Intuitively, the equilibrium represents a stable state where no client can increase its long-term utility by unilaterally changing its noise level. Each client has reached an optimal privacy-performance balance, so any deviation, either by adding or reducing noise, would lower its own utility. 

Assume we have the following assumptions:
\begin{assumption}[State/actions/strategies.]\label{assumption:1}
$\mathcal W$ is a Polish space, and $\mathcal B(\mathcal W)$ denote its Borel $\sigma$-algebra. The client set is $\mathcal C=\{c_1,\ldots,c_N\}$. The action set $\Sigma\subset[0,1]$ is finite. A stationary Markov mixed strategy for $c_i$ is a measurable map $\pi_i:\mathcal W\to\mathsf{Prob}(\Sigma)$. Conditional on $W_g$, each client mix independently use its own mixed strategy.
\end{assumption}

\begin{assumption}[Aggregative structure.]\label{assumption:2}
Each client’s update summary $Q_i:\mathcal W\times\Sigma\to\mathcal H$ is bounded and continuous into a fixed separable Hilbert space $\mathcal H$. Let the inverse-noise weights $w_i(\sigma)$ be bounded and continuous (e.g., $w_i(\sigma)=\frac{1/(\sigma_i+\tau)}{\sum_{j=1}^N1/(\sigma_j+\tau)}$ with $\tau>0$). Define
$\Gamma(W_g,\sigma)\ :=\ \sum_{i=1}^N w_i(\sigma)\,Q_i(W_g,\sigma_i)\ \in\ \mathcal H$.
\end{assumption}

\begin{assumption}[Utilities and bound.]\label{assumption:3}
For each $i$,\\
$U_i(W_g,\sigma)\ =\ \mathcal{G}_i\!\big(W_g,\Gamma(W_g,\sigma)\big)\ +\ \gamma_i\,\mathcal{P}_i(\sigma_i)$,where $\mathcal{G}_i$ and $\mathcal{P}_i$ are bounded and continuous. There exist $e_0,e_1<\infty$ and a Lyapunov function $V:\mathcal W\to[1,\infty)$ such that$|U_i(W_g,\sigma)|\ \le\ e_0 + e_1\,V(W_g),\forall c_i,\ (W_g,\sigma)$.
\end{assumption}

\begin{assumption}[Feller dynamics.]\label{assumption:4}
The transition kernel $\mathcal F(\cdot\mid W_g,\Gamma)$ is Feller in $(W_g,\Gamma)$; i.e., for every bounded continuous $\varphi:\mathcal W\to\mathbf{R}$,
$\int \varphi(W_g')\,\mathcal F(\mathrm dW_g'\mid W_g^n,\Gamma^n)\ \to\ \int \varphi(W_g')\,\mathcal F(\mathrm dW_g'\mid W_g,\Gamma)$, whenever $(W_g^n,\Gamma^n)\to(W_g,\Gamma)$.
\end{assumption}

\begin{assumption}[Foster–Lyapunov drift.]\label{assumption:5}
There exist $a\in(0,1)$, $b<\infty$, a petite set $C\subset\mathcal W$, and the $V$ from~\Cref{assumption:3} such that for all $W_g$ and all aggregates in the range of $\Gamma(W_g,\cdot)$,
$\int V(W_g')\,\mathcal F(\mathrm dW_g'\mid W_g,\Gamma)\ \le\ a\,V(W_g)\ +\ b\,\mathbf 1_C(W_g)$.
\end{assumption}

Under ~\Cref{assumption:1}–~\Cref{assumption:5}, the game admits at least one such equilibrium:

\begin{theorem}[Existence of stationary Markov equilibrium]\label{theorem:existence}
Assume~\Cref{assumption:1}–~\Cref{assumption:5} hold. Then there exists a stationary Markov equilibrium $\pi^*$ satisfying \cref{eq:equilibrium} for all $c_i\in\mathcal C$.
\end{theorem}

\begin{proof}
\textbf{Step 1: Profile-induced dynamics and long-run average.}
This step turns the dynamic game under a fixed stationary profile into a stationary Markov chain and expresses average payoffs as an integral.

For any stationary profile $\pi$, define the induced kernel on states:
\begin{equation}\label{eq:induced}
\mathcal F^\pi(\mathrm dW_g'\mid W_g)\ :=\ \int_{\Sigma^N} \mathcal F\!\big(\mathrm dW_g'\mid W_g,\Gamma(W_g,\sigma)\big)\,
\Big(\bigotimes_{j\in\mathcal \mathcal{S}}\pi_j(W_g)\Big)(\mathrm d\sigma).
\end{equation}
By~\Cref{assumption:2}–~\Cref{assumption:4} the map $(W_g,\pi)\mapsto \mathcal{F}^\pi(\cdot\mid W_g)$ is well-defined and Feller.
By the Foster–Lyapunov drift in~\Cref{assumption:5}, the Markov chain with kernel $\mathcal F^\pi$ is positive Harris recurrent, hence admits at least one invariant probability measure $\mu^\pi$ on $(\mathcal W,\mathcal B(\mathcal W))$.
Define the occupation measure on $\mathcal W\times\Sigma^N$ by:
\begin{equation}\label{eq:occ}
\eta^\pi(\mathrm dW_g,\mathrm d\sigma)\ :=\ \mu^\pi(\mathrm dW_g)\,\Big(\bigotimes_{j\in\mathcal C}\pi_j(W_g)\Big)(\mathrm d\sigma).
\end{equation}
With the bound in~\Cref{assumption:3}, the long-term average utility equals the linear functional:
\begin{equation}\label{eq:avg}
\overline U_i(\pi)=\int_{\mathcal W\times\Sigma^N} U_i(W_g,\sigma)\,\eta^\pi(\mathrm dW_g,\mathrm d\sigma).
\end{equation}

\medskip
\textbf{Step 2: Feasible occupation measures and best responses.}
This step describes all long-term behaviors client $c_i$ can induce with stationary Markov strategies and shows the average-reward best response is well behaved. 
For any stationary profile $(\pi_i,\pi_{-i})$, define the induced kernel:
\begin{equation}
\begin{split}
\mathcal F^{(\pi_i,\pi_{-i})}(\mathrm dW_g'\mid W_g)
:= & \int_{\Sigma^N} \mathcal F\!\big(\mathrm dW_g'\mid W_g,\Gamma(W_g,\sigma)\big)\,
\\
& \Big(\bigotimes_{j\in\mathcal C}\pi_j(W_g)\Big)(\mathrm d\sigma).
\end{split}
\end{equation}

By the uniform Foster--Lyapunov drift in~\Cref{assumption:5}, for all $W_g$ we have:\\
\begin{align*}
& \int V(W_g')\,\mathcal F^{(\pi_i,\pi_{-i})}(\mathrm dW_g'\mid W_g)\ \\
&= \int_{\Sigma^N}\!\Big[\int V(W_g')\,\mathcal F(\mathrm dW_g'\mid W_g,\Gamma(W_g,\sigma))\Big]\,
(\otimes_{j}\pi_j(W_g))(\mathrm d\sigma)\\
&\le \int_{\Sigma^N}\!\big[a\,V(W_g)+b\,\mathbf 1_C(W_g)\big]\,
(\otimes_{j}\pi_j(W_g))(\mathrm d\sigma)\\
&= a\,V(W_g)+b\,\mathbf 1_C(W_g).
\end{align*}

Hence $\mathcal F^{(\pi_i,\pi_{-i})}$ satisfies the same drift with $(V,a,b,\mathcal{S})$.
By the Foster--Lyapunov theorem, the chain with kernel
$\mathcal F^{(\pi_i,\pi_{-i})}$ is positive Harris recurrent and therefore admits at least one invariant
probability measure on $(\mathcal W,\mathcal B(\mathcal W))$; denote any such invariant law by
$\mu^{(\pi_i,\pi_{-i})}$.

Then, the associated occupation measure on $\mathcal W\times\Sigma^N$ is:
\begin{equation}\label{eq:occ-def}
\eta^{(\pi_i,\pi_{-i})}(\mathrm dW_g,\mathrm d\sigma)
\;:=\;
\mu^{(\pi_i,\pi_{-i})}(\mathrm dW_g)\,
\Big(\bigotimes_{j\in\mathcal C}\pi_j(W_g)\Big)(\mathrm d\sigma).
\end{equation}

Fix $c_i$ and others $\pi_{-i}$. Consider the feasible set of occupation measures:
\begin{equation}
\mathsf M_i(\pi_{-i})
\;:=\;
\big\{\eta^{(\pi_i,\pi_{-i})}\ :\ \pi_i:\mathcal W\to\mathsf{Prob}(\Sigma)\ \big\}.
\end{equation}

With ~\Cref{assumption:1}–~\Cref{assumption:5}, we have: $\mathsf M_i(\pi_{-i})$ is nonempty (existence of $\mu^{(\pi_i,\pi_{-i})}$), convex (independent mixing makes $\eta$ affine in $\pi_i$), and compact in the weak topology. Moreover, the correspondence $\pi_{-i}\mapsto\mathsf M_i(\pi_{-i})$ is upper hemicontinuous. 

Since $\eta\mapsto\int U_i\,\mathrm d\eta$ is linear continuous, Berge’s maximum theorem implies the average-reward best-response correspondence:
\begin{equation}
\mathrm{BR}_i(\pi_{-i})\ :=\ \arg\max_{\pi_i}\ \overline U_i\big((\pi_i,\pi_{-i})\big)
\ =\ \arg\max_{\eta\in\mathsf M_i(\pi_{-i})}\ \int U_i\,\mathrm d\eta
\end{equation}
is nonempty, convex-valued, and upper hemicontinuous.

\medskip
\textbf{Step 3: Kakutani fixed point.}
This step models clients’ best responses and applies Kakutani to obtain a stationary Markov equilibrium.

Let the strategy space be $\mathcal X:=\prod_{i=1}^N\{\pi_i:\mathcal W\to\mathsf{Prob}(\Sigma)\}$ endowed with the product of pointwise-weak topologies. By Tychonoff, $\mathcal X$ is compact and convex.
Define the product best-response correspondence $\mathrm{BR}(\pi):=\prod_{i=1}^N \mathrm{BR}_i(\pi_{-i})$. By Step~2, $\mathrm{BR}$ has nonempty convex values and a closed graph (hence is upper hemicontinuous).
By Kakutani’s fixed-point theorem, there exists $\pi^*\in\mathcal X$ with $\pi^*\in\mathrm{BR}(\pi^*)$. 
By definition, for every client $c_i$ and every stationary Markov alternative $\pi_i':\mathcal W\to\mathsf{Prob}(\Sigma)$,
\begin{equation}
\overline{U}_i(\pi^*)\ \ge\ \overline{U}_i\big((\pi_i',\pi_{-i}^*)\big).
\end{equation}
which is precisely~\cref{eq:equilibrium}.
Therefore, $\pi^*$ is a stationary Markov equilibrium.
\end{proof}
\balance
\section{Experimental Details}
\label{appendix:models_detail}
This section provides a detailed description of the large language models and datasets used in~\Cref{sec:exp_results}.


\noindent\textbf{Models.}
\begin{itemize}[leftmargin=*,noitemsep]
    \item \textbf{TinyLlama}~\cite{zhang2024tinyllama}:
    a decoder-only transformer with 1.1B parameters. The model pre-trained on 3 trillion tokens.
    \item \textbf{GPT2-Large}~\cite{radford2019language}:
    a transformer-based language model created and released by OpenAI with 774M-parameter. The model is a pretrained model on English language using a causal language modeling.
\end{itemize}

\noindent\textbf{Datasets.}
\begin{itemize}[leftmargin=*,noitemsep]
    \item \textbf{AGNews}~\cite{zhang2015character}:
    4-class news topic classification (World, Sports, Business, Sci/Tech). Each example contains a title and a brief description describing a recent event/story across general-news topics.
    \item \textbf{DBpedia}~\cite{zhang2015character}:
    14-class ontology classification constructed from DBpedia. Each example provides an entity title and its abstract, labeled by an ontology category (e.g., company, film, place), written in factual third-person prose.
    \item \textbf{20Newsgroups}~\cite{lang1995newsweeder}:
    20-class topic classification on Usenet posts. Each example is a user-authored message, featuring informal style, quotes, headers/footers, and topic drift typical of discussion forums.
\end{itemize}

\noindent\textbf{Benchmarks}
\label{appendix:benchmarks}
\textbf{FedIT}~\cite{zhang2024towards} aggregates client LoRA updates with FedAvg by averaging $A$ and $B$ separately. \textbf{Flora}~\cite{wang2024flora} aggregates by stacking $A$ and $B$, which supports heterogeneous LoRA ranks and mitigates averaging bias. We also include \textbf{FedIT+INA} and \textbf{Flora+INA}, which add individual noise adaptation (INA) to FedIT and Flora, respectively. \textbf{FedMT+INA}~\cite{mao2024game} rewards clients who reduce noise through monetary payments transferred from those who increase noise. Because direct payments are impractical here, we substitute payments with aggregation weights to enable a fair comparison.

\noindent\textbf{Implementation Details.}
\label{appendix:settings}
LoRA uses rank $r=16$, scaling factor $\alpha=32$, dropout $=0.1$, and learning rate $=2\times 10^{-4}$~\cite{lotfi2024unlocking,wang2024flora}. For TinyLlama, LoRA adapters are inserted into \texttt{q\_proj} and \texttt{v\_proj}; for GPT2-Large, adapters are injected into \texttt{c\_attn}.
All the experiments are executed in parallel on two servers, each with 4~$\times$ NVIDIA
Tesla V100 GPUs (32GB). 

\end{document}